\documentclass[english,10pt,twocolumn,letterpaper]{article}
\usepackage[T1]{fontenc}
\usepackage[latin9]{inputenc}
\usepackage{array}
\usepackage{booktabs}
\usepackage{mathtools}
\usepackage{multirow}
\usepackage{amsmath}
\usepackage{amsthm}
\usepackage{amssymb}
\usepackage{graphicx}
\PassOptionsToPackage{normalem}{ulem}
\usepackage{ulem}

\makeatletter

\providecommand{\tabularnewline}{\\}

\theoremstyle{plain}
\newtheorem{thm}{\protect\theoremname}
\theoremstyle{plain}
\newtheorem*{thm*}{\protect\theoremname}

\usepackage[pagenumbers]{wacv} 

\usepackage[pagebackref,breaklinks,colorlinks]{hyperref}

\usepackage[capitalize]{cleveref}
\crefname{section}{Sec.}{Secs.}
\Crefname{section}{Section}{Sections}
\Crefname{table}{Table}{Tables}
\crefname{table}{Tab.}{Tabs.}

\def\wacvPaperID{2298} 
\def\confName{WACV}
\def\confYear{2024}

\makeatother

\usepackage{babel}
\providecommand{\theoremname}{Theorem}

\begin{document}
\global\long\def\do{\mathrm{do}}%
\global\long\def\Model{\text{RDM}}%
\global\long\def\argmax#1{\underset{#1}{\text{argmax }}}%
\global\long\def\argmin#1{\underset{#1}{\text{argmin }}}%
\global\long\def\ModelERM{\text{ERM}}%
\global\long\def\Var{\text{\ensuremath{\mathbb{V}}}}%
\global\long\def\Max{\text{max}}%
\global\long\def\Min{\text{min}}%
\global\long\def\MMD{\text{MMD}}%
\global\long\def\ModelVREx{\text{V-REx}}%
\global\long\def\ModelREx{\text{REx}}%
\global\long\def\ColoredMNIST{\text{ColoredMNIST}}%
\global\long\def\MNIST{\text{MNIST}}%
\global\long\def\dist{\text{dist}}%
\global\long\def\CausIRL{\text{CausIRL}}%
\global\long\def\CORAL{\text{CORAL}}%
\global\long\def\MMDAAE{\text{MMD-AAE}}%
\global\long\def\DRO{\text{DRO}}%
\global\long\def\IRM{\text{IRM}}%
\global\long\def\DomainBed{\text{DomainBed}}%

\title{Domain Generalisation via Risk Distribution Matching}
\author{Toan Nguyen, Kien Do, Bao Duong, Thin Nguyen\\
Applied Artificial Intelligence Institute, Deakin University, Australia\\
 \texttt{\small{}$\left\{ \text{s222165627,\,k.do,\,duongng,\,thin.nguyen}\right\} $@deakin.edu.au}}
\maketitle
\begin{abstract}
We propose a novel approach for domain generalisation (DG) leveraging
risk distributions to characterise domains, thereby achieving domain
invariance. In our findings, risk distributions effectively highlight
differences between training domains and reveal their inherent complexities.
In testing, we may observe similar, or potentially intensifying in
magnitude, divergences between risk distributions. Hence, we propose
a compelling proposition: Minimising the divergences between risk
distributions across training domains leads to robust invariance for
DG. The key rationale behind this concept is that a model, trained
on domain-invariant or stable features, may consistently produce similar
risk distributions across various domains. Building upon this idea,
we propose \textbf{\uline{R}}isk \textbf{\uline{D}}istribution
\textbf{\uline{M}}atching ($\Model$). Using the maximum mean discrepancy
(MMD) distance, $\Model$ aims to minimise the variance of risk distributions
across training domains. However, when the number of domains increases,
the direct optimisation of variance leads to linear growth in MMD
computations, resulting in inefficiency. Instead, we propose an approximation
that requires only one MMD computation, by aligning just two distributions:
that of the worst-case domain and the aggregated distribution from
all domains. Notably, this method empirically outperforms optimising
distributional variance while being computationally more efficient.
Unlike conventional DG matching algorithms, $\Model$ stands out for
its enhanced efficacy by concentrating on scalar risk distributions,
sidestepping the pitfalls of high-dimensional challenges seen in feature
or gradient matching. Our extensive experiments on standard benchmark
datasets demonstrate that $\Model$ shows superior generalisation
capability over state-of-the-art DG methods.
\end{abstract}

\section{Introduction\label{sec:Introduction}}

In recent years, deep learning (DL) models have witnessed remarkable
achievements and demonstrated super-human performance on training
distributions~\cite{lecun2015deep}. Nonetheless, this success is
accompanied by a caveat - deep models are vulnerable to distributional
shifts and exhibit catastrophic failures to unseen \emph{out-of-domain}
data~\cite{lu2021invariant,dittadi2020transfer}. Such limitations
hinder the widespread deployment of DL systems in real-world applications,
where \emph{domain difference} can be induced by several factors,
such as spurious correlations~\cite{arjovsky2019invariant} or variations
in location or time~\cite{shankar2021image}.

\begin{figure}
\begin{centering}
\includegraphics[width=1\columnwidth]{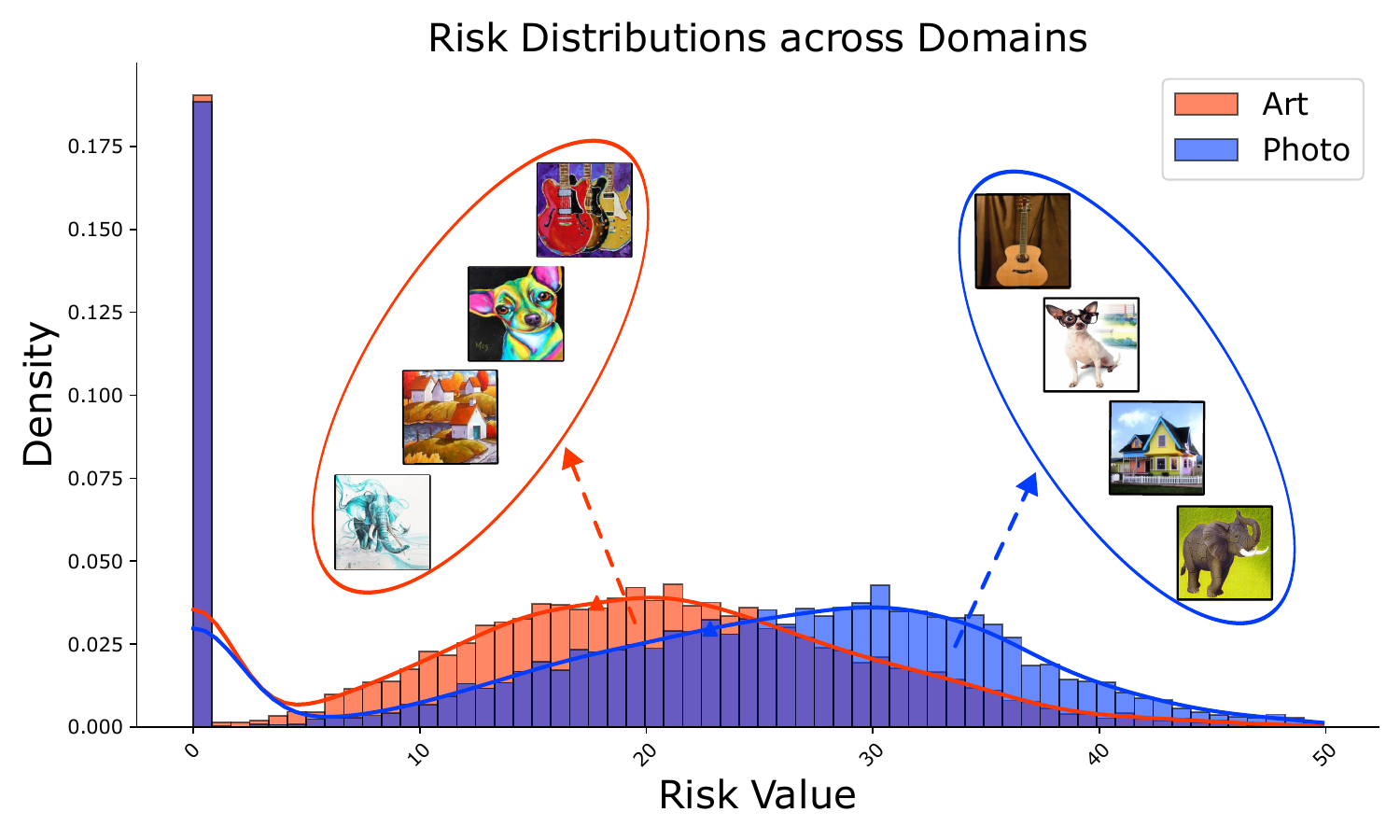}
\par\end{centering}
\caption{Risk distributions derived from training with $\protect\ModelERM$
for the ``Art'' and ``Photo'' domains on the validation set of
PACS dataset. Beyond low-risk samples, which may resemble training
data, the ``Photo'' domain generally exhibits a larger distribution
of risk values\emph{ }compared to ``Art''\emph{, }hinting at an
inherent complexity in learning ``Photo'' samples. The figure indicates
our motivation that \emph{risk distributions} can effectively highlight
differences between domains.\label{fig:teaser-image}}

\end{figure}

In light of these challenges, domain generalisation (DG) aims to produce
models capable of generalising to \emph{unseen} target domains by
leveraging data from diverse sets of training domains or environments~\cite{muandet2013domain}.
An effective approach involves exploring and establishing \emph{domain
invariance}~\cite{li2018domain2}\emph{, }with the expectation that
these invariances will similarly apply to related, yet distinct, test
domains. To this end, prevailing research focuses on characterising
domains through sample representation~\cite{li2018domain,muandet2013domain}.
The objective is to seek for domain-invariant features by aligning
the distributions of hidden representations across various domains.
$\CORAL$~\cite{sun2016deep} trains a non-linear transformation
that can align the second-order statistics of representations across
different layers within deep networks. More, $\CausIRL$~\cite{chevalley2022invariant}
aims to match representation distributions that have been intervened
upon the spurious factors. While these methods show promise, they
can face multiple challenges with the curse of dimensionality~\cite{bellman1966dynamic,indyk1998approximate}.
The sparsity of high-dimensional representation spaces can lead to
unreliable estimates of statistical properties, which in turn affects
the quality of distribution matching techniques. Also, high-dimensional
representations may contain many irrelevant or redundant dimensions,
which can introduce noise to the true underlying similarities or differences
between distributions. As dimensionality rises, computational complexity
intensifies, reducing the efficacy of these methods~\cite{muja2014scalable}.
Such challenges similarly present in DG methods that utilise gradients
for domain alignment~\cite{shi2022gradient,rame2022fishr}.

In this paper, we propose to utilise \emph{scalar risk distributions}
as a means to characterise domains, leading to successfully exploring
and enforcing domain invariance. Our research reveals that risk distributions
can be a reliable indicator of \emph{domain variation} as it effectively
highlights differences between training domains. In Figure~\ref{fig:teaser-image},
we present a visual evidence through histograms, contrasting the risk
distributions between the ``Art'' and ``Photo'' domains on the
validation set of PACS dataset~\cite{li2017deeper}, derived from
training with Empirical Risk Minimisation (ERM)~\cite{vapnik1999overview}.
The ``Photo'' domain generally exhibits a larger distribution of
scalar risks than that of ``Art''. This suggests an inherent complexity
in learning ``Photo'' samples, or possibly due to a more limited
training dataset compared to ``Art''. During the testing phase,
similar divergences between risk distributions may emerge, potentially
intensifying in magnitude. Hence, we propose a compelling proposition:
by \emph{minimising the divergences between risk distributions across
training domains}, we can achieve robust invariance for DG. The underlying
rationale for this concept is that a model, when learning domain-invariant
and stable features, tends to produce consistent risk distributions
across domains.

Building upon this idea, we propose a novel matching approach for
DG, namely \textbf{\emph{\uline{R}}}\emph{isk }\textbf{\emph{\uline{D}}}\emph{istribution
}\textbf{\emph{\uline{M}}}\emph{atching} ($\Model$). $\Model$'s
objective is to minimise the v\emph{ariance of risk distributions}
across all training domains. Inspired by~\cite{muandet2013domain},
we redefine the distributional variance metric to focus specifically
on risk distributions and propose to compute it via the maximum mean
discrepancy ($\MMD$) distance~\cite{gretton2012kernel}. However,
when the number of training domains increases, directly optimising
the variance induces a linear growth in $\MMD$ computations, reducing
efficiency. Instead, we propose an approximation that requires only
\emph{one MMD computation} via aligning just two distributions: that
of the \emph{worst-case }(or worst-performing)\emph{ }domain and the
aggregated distribution from all domains. Empirically, this approach
outperforms optimising distributional variance while significantly
reducing computational complexity. Unlike prevailing matching algorithms,
$\Model$ can address the high-dimensional challenges and further
improve efficacy by exclusively focusing on scalar risk distributions.
Notably, our empirical studies show that $\Model$ even exhibits enhanced
generalisation while being more convenient to optimise. We summarise
our contributions below: 
\begin{itemize}
\item We propose $\Model$, a novel and efficient matching method for DG,
based on our two hypotheses: i) risk distribution disparities offer
insightful cues into domain variation; ii) reducing these divergences
fosters a generalisable and invariant feature-learning predictor. 
\item We re-conceptualise the distributional variance metric to exclusively
focus on risk distributions, with an objective to minimise it. We
further provide an approximate version that aligns only the risk distribution
of the worst-case domain with the aggregate from all domains, improving
both performance and efficiency.
\item Through extensive experiments on standard benchmark datasets, we empirically
show that $\Model$ consistently outperforms state-of-the-art DG methods,
showcasing its remarkable generalisation capability.
\end{itemize}

\section{Related Work\label{sec:Related-Work}}

\paragraph*{Domain Generalisation (DG)}

DG aims to develop models that can generalise well on unseen target
domains by leveraging knowledge from multiple source domains. Typical
DG methods include domain alignment~\cite{muandet2013domain,ben2010theory,li2018domain2},
meta learning~\cite{li2018learning,balaji2018metareg}, data augmentation~\cite{zhou2021mixstyle,zhang2017mixup},
disentangled representation learning~\cite{sun2021recovering,peng2019domain},
robust optimisation~\cite{sagawa2019distributionally,buhlmann2020invariance}
and causality-based methods~\cite{krueger2021out,eastwood2022probable,nguyen2023causal}.
Our proposed method $\Model$ is related to domain alignment, striving
for \emph{domain invariance} to enhance OOD generalisation. Existing
research focuses on characterising domains through sample representations
and aligning their distributions across domains to achieve domain-invariant
features~\cite{albuquerque2019generalizing,muandet2013domain}. $\CORAL$~\cite{sun2016deep}
matches mean and variance of representation distributions, while $\MMDAAE$~\cite{li2018domain}
and FedKA~\cite{sun2023feature} consider matching all moments via
the maxmimum mean discrepancy ($\MMD$) distance~\cite{gretton2012kernel}.
Other methods promote domain invariance by minimising contrastive
loss~\cite{chen2020simple} between representations sharing the same
labels~\cite{mahajan2021domain,motiian2017unified}. Many studies
bypass the representation focus, instead characterising domains via
gradients and achieving invariance by reducing inter-domain gradient
variance~\cite{shi2022gradient,wang2023sharpness,rame2022fishr}.

Despite their potential, aligning these high-dimensional distributions
may be affected by data sparsity, diversity, and high computational
demands~\cite{bellman1966dynamic,indyk1998approximate}. Unlike these
methods, $\Model$ offers enhanced efficacy by focusing on \emph{scalar
risk distributions}, overcoming the high-dimensional challenges. Further,
$\Model$ adopts a novel strategy by efficiently aligning only two
distributions: that of the worst-case domain with the aggregate from
all domains. From our experiments, $\Model$ generally exhibits better
generalisation performance while being more convenient to optimise
compared to competing matching techniques. To the best of our knowledge,
the incarnation of risk distributions for domain matching in $\Model$
is novel and sensible.

\paragraph*{Distribution matching}

Distribution matching has been an important topic with a wide range
of applications in machine learning such as DG~\cite{li2018domain,sun2016deep},
domain adaptation~\cite{wang2018visual,chen2019progressive}, generative
modelling~\cite{li2017mmd,li2015generative}. Early methods, like
the MMD distance~\cite{gretton2012kernel}, leverage kernel-based
approaches to quantify the distance between distributions, laying
the foundation for many subsequent DG techniques~\cite{li2018domain,muandet2013domain}.
Further advancements have explored optimal transport methods, like
the Wasserstein distance~\cite{memoli2011gromov,arjovsky2017wasserstein},
which provides a geometrically intuitive means to compare distributions.
Other metrics, such as the Kullback-Leibler~\cite{kullback1951information}
or Jensen-Shannon~\cite{fuglede2004jensen,endres2003new} divergences,
can serve to measure the divergence between distributions and may
require additional parameters for estimating the density ratio between
the two distributions~\cite{sugiyama2012density}. In this paper,
we utilise the $\MMD$ distance to align risk distributions. Its inherent
advantages include an analytical measure of the divergence between
distributions without relying on distribution densities, and its non-parametric
nature~\cite{gretton2012kernel}. Alternative DG methods augment
data by utilising distribution matching and style transfer to generate
semantic-preserving samples~\cite{zhang2022exact,zhou2021mixstyle}.
Our method differs as we emphasise \emph{domain invariance via aligning
risk distributions}, rather than augmenting representation distributions.

\paragraph*{Invariance and Causality in DG}

Causal methods in DG assume that the causal mechanism of the target
given causal input features is invariant while non-causal features
may change across domains~\cite{arjovsky2019invariant,peters2016causal,krueger2021out}.
Based on this assumption, methods establish domain invariance to recover
the causal mechanism, thereby improving generalisation. ICP~\cite{peters2016causal}
has shown that the causal predictor has an invariant distribution
of residuals in regression models, however, is not suitable for deep
learning. EQRM~\cite{eastwood2022probable} and REx~\cite{krueger2021out}
leverage the invariance in the \emph{average} risks over samples across
domains. In contrast to above methods, we consider matching \emph{entire
risk distributions }over samples\emph{ }across domains\emph{, }which,
as our experiments demonstrate, is more powerful and enhances generalisation
capability.

\section{Preliminaries\label{sec:Preliminaries}}

Domain generalisation (DG) involves training a classifier $f$ on
data composed of multiple training domains (also called environments)
so that $f$ can perform well on unseen domains at test time. Mathematically,
let $\mathcal{D}=\left\{ \mathcal{D}_{1},...,\mathcal{D}_{m}\right\} $
denote the training set consisting of $m$ different domains/environments,
and let $\mathcal{D}_{e}:=\left\{ \left(x_{e}^{i},y_{e}^{i}\right)\right\} _{i=1}^{n_{e}}$
denote the training data belonging to domain $e$ ($1\leq e\leq m$).
Given a loss function $\ell$, the risk of a particular domain sample
$(x_{e}^{i},y_{e}^{i})$ is denoted by $R_{e}^{i}:=\ell\left(f\left(x_{e}^{i}\right),y_{e}^{i}\right)$,
and the expected risk $\overline{R}_{e}$ of domain $e$ is defined
as: 
\begin{align}
\overline{R}_{e} & :=\mathbb{E}_{\left(x_{e},y_{e}\right)\sim\mathcal{D}_{e}}\left[\ell\left(f\left(x_{e}\right),y_{e}\right)\right]=\mathbb{E}_{\mathcal{D}_{e}}\left[R_{e}\right]\label{eq:domain_expected_risk}
\end{align}

A common approach to train $f$ is Empirical Risk Minimisation (ERM)~\cite{vapnik1999overview}
which minimises the expected risks across all training domains. Its
loss function, denoted by $\mathcal{L}_{\text{ERM}}$, is computed
as follows: 
\begin{align}
\mathcal{L}_{\text{\ensuremath{\ModelERM}}} & =\mathbb{E}_{e\sim\mathcal{E}}\mathbb{E}_{\left(x_{e},y_{e}\right)\sim\mathcal{D}_{e}}\left[\ell\left(f\left(x_{e}\right),y_{e}\right)\right]\label{eq:erm_loss_1}\\
 & =\mathbb{E}_{e\sim\mathcal{E}}\left[\overline{R}_{e}\right]\label{eq:erm_loss_2}
\end{align}
where $\mathcal{E}:=\{1,...,m\}$ denotes the set of all domains.

\section{Risk Distribution Matching\label{sec:Methodology}}

A model $f$ trained via ERM often struggles with generalisation to
new test domains. This is because it tends to capture domain-specific
features~\cite{arjovsky2019invariant,nguyen2023causal}, such as
domain styles, to achieve low risks in training domains, rather than
focusing on domain-invariant or semantic features. To overcome this
issue, we present a novel training objective that bolsters generalisation
through \emph{domain} \emph{invariance}. Our goal requires utilising
a unique domain representative that both characterises each domain
and provides valuable insights into domain variation. Specifically,
we propose to leverage the \emph{distribution of risks over all samples
within a domain} (or shortly \emph{risk distribution}) as this representative.
Unlike other domain representatives, like latent representation or
gradient distributions~\cite{li2018domain,shi2022gradient}, the
risk distribution sidesteps high-dimensional challenges like data
sparsity and high computational demands~\cite{bellman1966dynamic,muja2014scalable}.
In essence, a model capturing stable, domain-invariant features may
consistently yield similar risk distributions across all domains.
In pursuit of invariant models, we propose \textbf{\emph{\uline{R}}}\emph{isk
}\textbf{\emph{\uline{D}}}\emph{istribution }\textbf{\emph{\uline{M}}}\emph{atching}
($\Model$), a novel approach for DG that reduces the divergences
between training risk distributions via \emph{minimising the distributional
variance across them}. 

Let $\mathcal{T}_{e}$ be the probability distribution over the risks
of all samples in domain $e$ (i.e., $\left\{ R_{e}^{i}\right\} _{i=1}^{n_{e}}$).
We refer to $\mathcal{T}_{e}$ as the \emph{risk distribution} of
domain $e$, the representative that effectively captures the core
characteristics of the domain. We denote $\mathbb{V}_{\mathbb{R}}\left(\left\{ \mathcal{T}_{1},...,\mathcal{T}_{m}\right\} \right)$
the distributional variance across the risk distributions $\left\{ \mathcal{T}_{1},...,\mathcal{T}_{m}\right\} $
in the real number space. We achieve our objective by minimising the
following loss function:

\begin{equation}
\mathcal{L}_{\text{final}}:=\ \mathcal{L}_{\text{ERM}}+\lambda\mathbb{V}_{\mathbb{R}}\left(\left\{ \mathcal{T}_{1},...,\mathcal{T}_{m}\right\} \right)\label{eq:general_variance_rdm_loss_1}
\end{equation}
where $\left(\lambda\geq0\right)$ is a coefficient balancing between
reducing the total training risks with enforcing invariance across
domains. $\lambda$ is set to 1 unless specified otherwise.

To compute $\mathbb{V}_{\mathbb{R}}\left(\left\{ \mathcal{T}_{1},...,\mathcal{T}_{m}\right\} \right)$,
we require a suitable representation for the \emph{implicit} risk
distribution $\mathcal{T}_{e}$ of domain $e$. Leveraging kernel
mean embedding~\cite{smola2007hilbert}, we express $\mathcal{T}_{e}$
as its embedding, $\mu_{\mathcal{T}_{e}}$, within a reproducing kernel
Hilbert space (RKHS) $\mathcal{H}$ using a feature map $\phi:\mathbb{R}\rightarrow\mathcal{H}$
below: 

\begin{align}
\mu_{\mathcal{T}_{e}} & \coloneqq\ \mathbb{E}_{R_{e}\sim\mathcal{T}_{e}}\left[\phi\left(R_{e}\right)\right]\label{eq:mean_embed}\\
 & =\ \mathbb{E}_{R_{e}\sim\mathcal{T}_{e}}\left[k\left(R_{e},\cdot\right)\right]\label{eq:mean_embed_eq2}
\end{align}
where a kernel function $k\left(\cdot,\cdot\right):\mathbb{R}\times\mathbb{R}\rightarrow\mathbb{R}$
is introduced to bypass the explicit specification of $\phi$. Assuming
the condition $\left(\mathbb{E}_{R_{e}\sim\mathcal{T}_{e}}\left(k\left(R_{e},R_{e}\right)\right)<\infty\right)$
holds, the mean map $\mu_{\mathcal{T}_{e}}$ remains an element of
$\mathcal{H}$~\cite{gretton2012kernel,li2018domain}. It is noteworthy
that for a \emph{characteristic} kernel $k$, the representation $\mu_{\mathcal{T}_{e}}$
within $\mathcal{H}$ is unique~\cite{muandet2013domain,gretton2012kernel}.
Consequently, two distinct risk distributions $\mathcal{T}_{u}$ and
$\mathcal{T}_{v}$ for any domains $u,v$ respectively have different
kernel mean embeddings in $\mathcal{H}$. In this work, we use the
RBF kernel, a well-known characteristic kernel defined as $k\left(x,x'\right):=\text{exp\ensuremath{\left(-\frac{1}{2\sigma}\left\Vert x-x'\right\Vert ^{2}\right)}}$,
where $\sigma>0$ is the bandwidth parameter.

With the unique representation of $\mathcal{T}_{e}$ established,
our objective becomes computing the distributional variance between
risk distributions within $\mathcal{H}$, represented by $\mathbb{V}_{\mathbb{\mathcal{H}}}\left(\left\{ \mathcal{T}_{1},...,\mathcal{T}_{m}\right\} \right)$.
Inspired by~\cite{muandet2013domain}, we redefine the variance metric
to focus specifically on risk distributions across multiple domains
below:

\begin{equation}
\mathbb{V}_{\mathcal{H}}\left(\left\{ \mathcal{T}_{1},...,\mathcal{T}_{m}\right\} \right):=\ \frac{1}{m}\sum_{e=1}^{m}\left\Vert \mu_{\mathcal{T}_{e}}-\mu_{\mathcal{T}}\right\Vert _{\mathcal{H}}^{2}\label{eq:dist_variance_1}
\end{equation}
where $\mathcal{T}=\frac{1}{m}\sum_{e=1}^{m}\mathcal{T}_{e}$ denotes
the probability distribution over the risks of all samples in the
entire training set, or equivalently, the set of all $m$ domains.
Meanwhile, $\mu_{\mathcal{T}_{e}}$ and $\mu_{\mathcal{T}}$ represent
the mean embedings of $\mathcal{T}_{e}$ and $\mathcal{T}$, respectively,
and are computed as in Eq.~\ref{eq:mean_embed}. Incorporating $\mathbb{V}_{\mathcal{H}}\left(\left\{ \mathcal{T}_{1},...,\mathcal{T}_{m}\right\} \right)$
into our loss function from Eq.~\ref{eq:general_variance_rdm_loss_1},
we get:

\begin{align}
\mathcal{L}_{\text{final}} & :=\ \mathcal{L}_{\text{ERM}}+\lambda\mathbb{V}_{\mathbb{\mathcal{H}}}\left(\left\{ \mathcal{T}_{1},...,\mathcal{T}_{m}\right\} \right)\label{eq:general_variance_rdm_loss_2}
\end{align}
Minimising $\mathbb{V}_{\mathbb{\mathcal{H}}}\left(\left\{ \mathcal{T}_{1},...,\mathcal{T}_{m}\right\} \right)$
in Eq.~\ref{eq:general_variance_rdm_loss_2} facilitates our objective
of equalising risk distributions across all domains, as proven by
the theorem below. 
\begin{thm}
~\cite{muandet2013domain} Given the distributional variance $\mathbb{V}_{\mathbb{\mathcal{H}}}\left(\left\{ \mathcal{T}_{1},...,\mathcal{T}_{m}\right\} \right)$
is calculated with a characteristic kernel $k$, $\mathbb{V}_{\mathcal{H}}\left(\left\{ \mathcal{T}_{1},...,\mathcal{T}_{m}\right\} \right)=0$
if and only if $\mathcal{T}_{1}=...=\mathcal{T}_{m}\left(=\mathcal{T}\right)$.
\end{thm}
\begin{proof}
Please refer to our appendix for the proof.
\end{proof}
In the next part, we present how to compute the distributional variance
using the Maximum Mean Discrepancy (MMD) distance~\cite{gretton2012kernel},
relying only on risk samples. Then, we propose an efficient approximation
of optimising the distributional variance, yielding improved empirical
performance.

\subsection{Maximum Mean Discrepancy}

For domain $e$, the squared norm, $\left\Vert \mu_{\mathcal{T}_{e}}-\mu_{\mathcal{T}}\right\Vert _{\mathcal{H}}^{2}$,
defined in Eq.~\ref{eq:dist_variance_1}, is identified as the squared
$\MMD$ distance~\cite{gretton2006kernel} between distributions
$\mathcal{T}_{e}$ and $\mathcal{T}$. It is expressed as follows:

\begin{align}
\MMD^{2}\left(\mathcal{T}_{e},\mathcal{T}\right)=\  & \left\Vert \mu_{\mathcal{T}_{e}}-\mu_{\mathcal{T}}\right\Vert _{\mathcal{H}}^{2}\\
=\  & \left\Vert \mathbb{E}_{R_{e}\sim\mathcal{T}_{e}}\left[\phi\left(R_{e}\right)\right]-\mathbb{E}_{R_{f}\sim\mathcal{T}}\left[\phi\left(R_{f}\right)\right]\right\Vert _{\mathcal{H}}^{2}\\
=\  & \mathbb{E}_{R_{e},R_{e}^{'}\sim\mathcal{T}_{e}}\left\langle \phi\left(R_{e}\right),\phi\left(R_{e}^{'}\right)\right\rangle \nonumber \\
 & -2\mathbb{E}_{R_{e}\sim\mathcal{T}_{e};R_{f}\sim\mathcal{T}}\left\langle \phi\left(R_{e}\right),\phi\left(R_{f}\right)\right\rangle \label{eq:MMD_3}\\
 & +\mathbb{E}_{R_{f},R_{f}^{'}\sim\mathcal{T}}\left\langle \phi\left(R_{f}\right),\phi\left(R_{f}^{'}\right)\right\rangle \nonumber 
\end{align}
where $\left\langle \cdot,\cdot\right\rangle $ denote the inner product
operation in $\mathcal{H}.$ Through the kernel trick, we can compute
these inner products via the kernel function $k$ without an explicit
form of $\phi$ below:

\begin{align}
\MMD^{2}\left(\mathcal{T}_{e},\mathcal{T}\right)=\  & \mathbb{E}_{R_{e},R_{e}^{'}\sim\mathcal{T}_{e}}k\left(R_{e},R_{e}^{'}\right)\nonumber \\
 & -2\mathbb{E}_{R_{e}\sim\mathcal{T}_{e};R_{f}\sim\mathcal{T}}k\left(R_{e},R_{f}\right)\label{eq:MMD_4}\\
 & +\mathbb{E}_{R_{f},R_{f}^{'}\sim\mathcal{T}}k\left(\mathcal{R}_{f},\mathcal{R}_{f}^{'}\right)\nonumber 
\end{align}

We reformulate our loss function in Eq.~\ref{eq:general_variance_rdm_loss_2}
to incorporate MMD as follows:

\begin{align}
\mathcal{L}_{\text{final}} & :=\ \mathcal{L}_{\text{ERM}}+\frac{\lambda}{m}\sum_{e=1}^{m}\MMD^{2}\left(\mathcal{T}_{e},\mathcal{T}\right)\label{eq:general_variance_rdm_mmd_loss_1}\\
 & =\ \mathcal{L}_{\text{ERM}}+\lambda\mathcal{L_{\Model}}\label{eq:general_variance_rdm_mmd_loss_2}
\end{align}
The loss function $\mathcal{L_{\Model}}$ involves minimising $\MMD^{2}\left(\mathcal{T}_{e},\mathcal{T}\right)$
for every domain $e$. Ideally, the distributional variance reaches
its lowest value at $0$ if $\MMD\left(\mathcal{T}_{e},\mathcal{T}\right)=0$,
equivalent to $\left(\mathcal{T}_{e}=\mathcal{T}\right)$~\cite{gretton2012kernel,gretton2006kernel},
across $e$ domains. The objective also entails aligning each individual
risk distribution, $\mathcal{T}_{e}$, with the aggregated distribution
spanning all domains, $\mathcal{T}$. With the characteristic RBF
kernel, it can be viewed as \emph{matching an infinite number of moments}
across all risk distributions. 

We emphasise our choice of $\MMD$ owing to its benefits for effective
risk distribution matching: i) $\MMD$ is an important member of the
Integral Probability Metric family~\cite{muller1997integral} that
offers an analytical solution facilitated through RKHS, and ii) $\MMD$
enjoys the property of quantifying the dissimilarity between two implicit
distributions via their finite samples in a non-parametric manner. 

\subsection{Further improvement of $\protect\Model$}

We find that effective alignment of risk distributions across $m$
domains can be achieved by matching the risk distribution of the \emph{worst-case
}domain, denoted as $w$, with the combined risk distribution of all
domains, offering an approximation to the optimisation of risk distributional
variance seen in Eq.~\ref{eq:general_variance_rdm_mmd_loss_1}. This
approximate version significantly reduces the $\MMD$ distances computation
in $\mathcal{L}_{\text{RDM}}$ from $O\left(m\right)$ to $O\left(1\right)$,
and further improves generalisation, as we demonstrate with empirical
evidence in Section~\ref{sec:Experiments}. 

Denote by $\left(w=\argmax{e\in\mathcal{E}}\overline{R}_{e}\right)$
the worst-case domain, i.e., the domain that has the largest expected
risk in $\mathcal{E}$. The approximate $\Model$'s loss,$\hat{\mathcal{L}}_{\text{RDM}}$,
is computed as follows: 
\begin{align}
\mathcal{\hat{L}}_{\Model} & =\MMD^{2}\left(\mathcal{T}_{w},\mathcal{T}\right)\label{eq:improved_rdm_1}\\
 & \approx\mathcal{L}_{\Model}
\end{align}

In our experiments, we observed only a small gap between $\mathcal{\hat{L}}_{\Model}$
and $\mathcal{L}_{\Model}$, while optimising $\mathcal{\hat{L}}_{\Model}$
proving to be more computationally efficient. The key insight emerges
from $\overline{R}_{e}$, the first moment (or mean) of $\mathcal{T}_{e}$.
Often, the average risk can serve as a measure of domain uniqueness
or divergence~\cite{krueger2021out,rame2022fishr}. Specifically,
a domain with notably distinct mean risk is more likely to diverge
greatly from other risk distributions. Under such circumstances, $\mathcal{\hat{L}}_{\Model}$
will be an upper-bound of $\mathcal{L_{\Model}}$, as shown by: $\mathcal{\mathcal{\mathcal{L}_{\Model}}}=\frac{1}{m}\sum_{e=1}^{m}\MMD^{2}\left(\mathcal{T}_{e},\mathcal{T}\right)\leq\frac{1}{m}\sum_{e=1}^{m}\MMD^{2}\left(\mathcal{T}_{w},\mathcal{T}\right)=\MMD^{2}\left(\mathcal{T}_{w},\mathcal{T}\right)=\mathcal{\hat{L}}_{\Model}$.
By optimising $\mathcal{\hat{L}}_{\Model}$, we can also potentially
decrease $\mathcal{L}_{\Model}$, thus aligning risk distributions
across domains effectively. More, $\mathcal{\hat{L}}_{\Model}$ drives
the model to prioritise the worst-case domain's optimisation. This
approach enhances the model's robustness to extreme training scenarios,
which further improves generalisation as proven in~\cite{sagawa2019distributionally,krueger2021out}.
These insights shed light on the superior performance of optimising
$\mathcal{\hat{L}}_{\Model}$ over $\mathcal{L}_{\Model}$. Therefore,
we opted to use $\mathcal{\hat{L}}_{\Model}$, simplifying our model's
training and further bolstering its OOD performance.

\section{Experiments\label{sec:Experiments}}

We evaluate and analyse $\Model$ using a synthetic $\ColoredMNIST$
dataset~\cite{arjovsky2019invariant} and multiple benchmarks from
the DomainBed suite~\cite{gulrajani2021in}. Each of our claims is
backed by empirical evidence in this section. Our source code to reproduce
results is available at:~\url{https://github.com/nktoan/risk-distribution-matching}

\subsection{Synthetic Dataset: ColoredMNIST}

We evaluate all baselines on a synthetic binary classification task,
namely $\ColoredMNIST$~\cite{arjovsky2019invariant}. This dataset
involves categorising digits (0-9) into two labels: ``zero'' for
0 to 4 range and ``one'' for 5 to 9 range, with each digit colored
either red or green. The dataset is designed to assess the generalisation
and robustness of baseline models against the influence of spurious
color features. The dataset contains two training domains, where the
chance of red digits being classified as ``zero'' is $80$\% and
$90$\%, respectively, while this probability decreases to only $10$\%
during testing. The goal is to train a predictor invariant to ``digit
color'' features, capturing only ``digit shape'' features.

Following~\cite{eastwood2022probable}, we employ a two-hidden-layer
MLP with 390 hidden units for all baselines. Optimised through the
Adam optimiser~\cite{kingma2015adam} at a learning rate of $0.0001$,
with a dropout rate of $0.2$, we train each algorithm for $600$
iterations with a batch size of 25,000. We repeat the experiment ten
times over different values of the penalty weight $\lambda$. We find
our matching penalty quite small, yielding optimal $\Model$'s performance
within the $\lambda$ range of $[1000,10000]$. We provide more details
about experimental settings in the supplementary material.

\begin{table}
\begin{centering}
\begin{tabular}{ccc}
\toprule 
\multirow{2}{*}{Algorithm} & \multicolumn{2}{c}{Initialisation}\tabularnewline
\cmidrule{2-3} \cmidrule{3-3} 
 & Rand.  & ERM\tabularnewline
\midrule
\midrule 
ERM  & 27.9$\pm$1.5  & 27.9$\pm$1.5\tabularnewline
GroupDRO  & 27.3$\pm$0.9  & 29.0$\pm$1.1\tabularnewline
IGA  & 50.7$\pm$1.4  & 57.7$\pm$3.3\tabularnewline
IRM  & 52.5$\pm$2.4  & 69.7$\pm$0.9\tabularnewline
VREx  & 55.2$\pm$4.0  & 71.6$\pm$0.5\tabularnewline
EQRM  & 53.4$\pm$1.7  & 71.4$\pm$0.4\tabularnewline
CORAL  & 55.3$\pm$2.8  & 65.6$\pm$1.1\tabularnewline
MMD  & 54.6$\pm$3.2  & 66.4$\pm$1.7\tabularnewline
\midrule 
$\Model$ (\emph{ours})  & \textbf{56.3$\pm$1.5}  & \textbf{72.4$\pm$1.0}\tabularnewline
\midrule 
Oracle  & \multicolumn{2}{c}{72.1$\pm$0.7}\tabularnewline
Optimum  & \multicolumn{2}{c}{75.0}\tabularnewline
\bottomrule
\end{tabular}
\par\end{centering}
\caption{$\protect\ColoredMNIST$ test accuracy where the best results are
marked as bold. Results of other methods are referenced from~\cite{eastwood2022probable}.
\label{tab:-CMNIST_result_mainpaper}}
\end{table}

\begin{figure*}[t]
\subfloat[ERM's histogram]{\begin{centering}
\includegraphics[width=1\columnwidth]{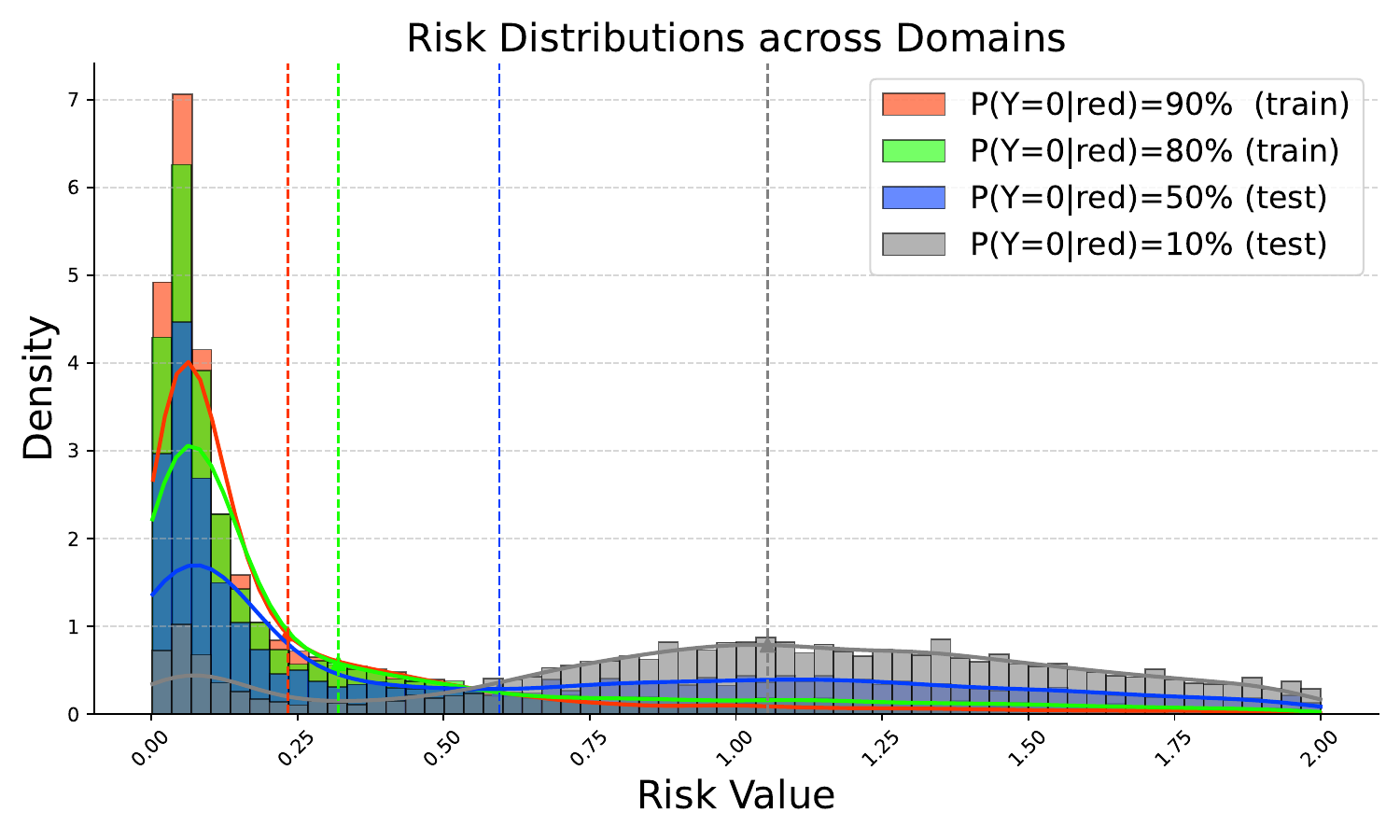} 
\par\end{centering}

}\hfill{}\subfloat[$\protect\Model$'s histogram]{\begin{centering}
\includegraphics[width=1\columnwidth]{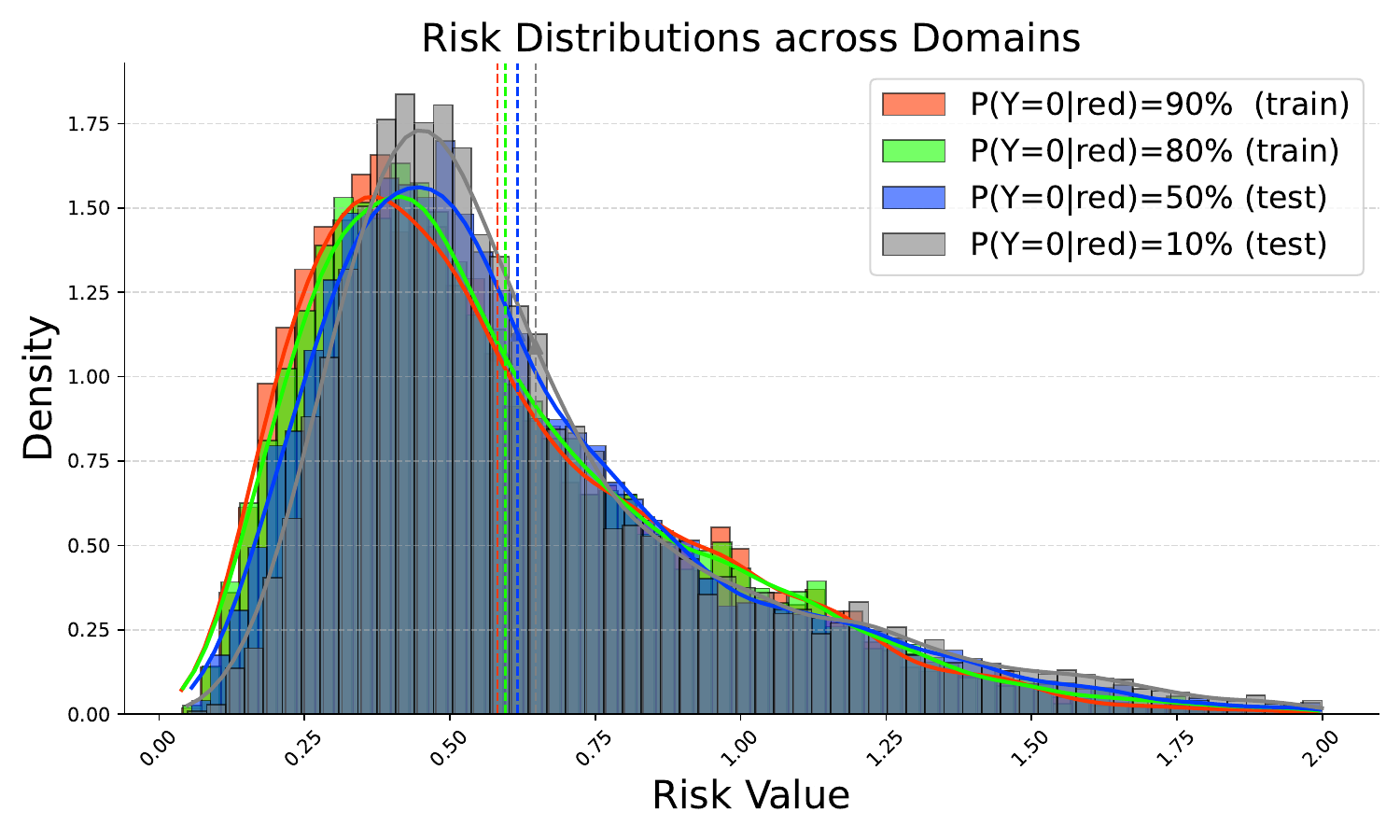} 
\par\end{centering}
}

\caption{Histograms with their KDE curves depicting the risk distributions
of $\protect\ModelERM$ and $\protect\Model$ across four domains
on $\protect\ColoredMNIST$. Vertical ticks denote the mean values
of all distributions. \label{fig:Histograms-of-risk-dists-CMNIST}}
\end{figure*}

We compare $\Model$ with $\ModelERM$ and three different types of
algorithms: robust optimisation (GroupDRO~\cite{sagawa2019distributionally},
IGA~\cite{koyama2020out}), causal methods learning invariance (IRM~\cite{arjovsky2019invariant},
VREx~\cite{krueger2021out}, EQRM~\cite{eastwood2022probable})
and representation distribution matching (MMD~\cite{li2018domain},
CORAL~\cite{sun2016deep}). All algorithms are run using two distinct
network configurations: (i) initialising the network randomly via
Xavier method~\cite{glorot2010understanding}; (ii) pre-training
the network with $\ModelERM$ for $400$ iterations prior to performing
the algorithms. Table~\ref{tab:-CMNIST_result_mainpaper} shows that
our proposed method $\Model$ surpasses all algorithms, irrespective
of the network configuration. $\Model$ exhibits improvements of $1.0$\%
and $6.8$\% over $\CORAL$, both without and with pre-trained $\ModelERM$,
respectively, underlining the effectiveness of aligning risk distributions
instead of high-dimensional representations. VREx and EQRM, which
pursue invariant predictors by equalising average training risks across
domains, demonstrate suboptimal performance compared to our approach.
This improvement arises from our consideration of the entire risk
distributions and the matching of all moments across them, which inherently
foster stronger invariance for DG. Notably, all methods experience
enhanced performance with $\ModelERM$ initialisation. $\Model$ even
excels beyond oracle performance ($\ModelERM$ trained on grayscale
digits with 50\% red and 50\% green) and converges towards optimality.

Figure~\ref{fig:Histograms-of-risk-dists-CMNIST} demonstrates histograms
with their KDE curves~\cite{parzen1962estimation} depicting the
risk distributions of $\ModelERM$ and $\Model$ across four domains.
The figure confirms our hypothesis that the disparities among risk
distributions could serve as a valuable signal of \emph{domain variation}.
$\ModelERM$'s histogram shows a clear difference between environments
with $90$\% and $80$\% chance of red digits labelled ``zero''
and those with only $50$\% or $10$\%. More, $\ModelERM$ tends to
overfit to training domains, which negatively impacts its generalisation
to test domains. Remarkably, $\Model$ effectively minimises the divergences
between risk distributions across all domains, including \emph{test
domains} \emph{with lower risks}. This also aligns with our motivation:
an invariant or stable feature-learning predictor, by displaying similar
risk distributions across domains, inherently boosts generalisation.

\subsection{DomainBed}

\begin{table*}[t]
\begin{centering}
\begin{tabular}{ccccccc}
\toprule 
Algorithm  & VLCS  & PACS  & OfficeHome  & TerraIncognita  & DomainNet  & Avg\tabularnewline
\midrule
\midrule 
ERM  & 77.5$\pm$0.4  & 85.5$\pm$0.2  & 66.5$\pm$0.3  & 46.1$\pm$1.8  & 40.9$\pm$0.1  & 63.3\tabularnewline
Mixup  & 77.4$\pm$0.6  & 84.6$\pm$0.6  & 68.1$\pm$0.3  & \textbf{47.9$\pm$0.8}  & 39.2$\pm$0.1  & 63.4\tabularnewline
MLDG  & 77.2$\pm$0.4  & 84.9$\pm$1.0  & 66.8$\pm$0.6  & 47.7$\pm$0.9  & 41.2$\pm$0.1  & 63.6\tabularnewline
GroupDRO  & 76.7$\pm$0.6  & 84.4$\pm$0.8  & 66.0$\pm$0.7  & 43.2$\pm$1.1  & 33.3$\pm$0.2  & 60.9\tabularnewline
IRM  & 78.5$\pm$0.5  & 83.5$\pm$0.8  & 64.3$\pm$2.2  & 47.6$\pm$0.8  & 33.9$\pm$2.8  & 61.6\tabularnewline
VREx  & 78.3$\pm$0.2  & 84.9$\pm$0.6  & 66.4$\pm$0.6  & 46.4$\pm$0.6  & 33.6$\pm$2.9  & 61.9\tabularnewline
EQRM  & 77.8$\pm$0.6  & 86.5$\pm$0.2  & 67.5$\pm$0.1  & 47.8$\pm$0.6  & 41.0$\pm$0.3  & 64.1\tabularnewline
Fish & 77.8$\pm$0.3  & 85.5$\pm$0.3  & 68.6$\pm$0.4  & 45.1$\pm$1.3  & 42.7$\pm$0.2  & 64.0\tabularnewline
Fishr & 77.8$\pm$0.1  & 85.5$\pm$0.4  & 67.8$\pm$0.1  & 47.4$\pm$1.6  & 41.7$\pm$0.0  & 64.0\tabularnewline
CORAL  & \textbf{78.8$\pm$0.6}  & 86.2$\pm$0.3  & \textbf{68.7$\pm$0.3}  & 47.6$\pm$1.0  & 41.5$\pm$0.1  & 64.6\tabularnewline
MMD  & 77.5$\pm$0.9  & 84.6$\pm$0.5  & 66.3$\pm$0.1  & 42.2$\pm$1.6  & 23.4$\pm$9.5  & 63.3\tabularnewline
\midrule 
$\Model$ (\emph{ours})  & 78.4$\pm$0.4  & \textbf{87.2$\pm$0.7}  & 67.3$\pm$0.4  & 47.5$\pm$1.0  & \textbf{43.4$\pm$0.3}  & \textbf{64.8}\tabularnewline
\bottomrule
\end{tabular}
\par\end{centering}
\caption{$\protect\DomainBed$ test accuracy where the best results are marked
as bold. Results of other methods are referenced from~\cite{eastwood2022probable,shi2022gradient}.
Model selection: training-domain validation set.\label{tab:DomainBed-results-mainpaper}}
\end{table*}

\paragraph*{Dataset and Protocol}

Following previous works~\cite{gulrajani2021in,eastwood2022probable},
we extensively evaluate all methods on five well-known DG benchmarks:
VLCS~\cite{fang2013unbiased}, PACS~\cite{li2017deeper}, OfficeHome~\cite{venkateswara2017deep},
TerraIncognita~\cite{beery2018recognition}, and DomainNet~\cite{peng2019moment}.
For a fair comparison, we reuse the training and evaluation protocol
in $\DomainBed$~\cite{gulrajani2021in}, including the dataset splits,
training iterations, and model selection criteria. Our evaluation
employs the leave-one-domain-out approach: each model is trained on
all domains except one and then tested on the excluded domain. The
final model is chosen based on its combined accuracy across all training-domain
validation sets.

\paragraph*{Implementation Details}

We use ResNet-50~\cite{he2016deep} pre-trained on ImageNet~\cite{olga2015imagenet}
as the default backbone. The model is optimised via the Adam optimiser
for $5,000$ iterations on every dataset. We follow~\cite{eastwood2022probable,krueger2021out}
to pre-train baselines with $\ModelERM$ for certain iterations before
performing the algorithms. Importantly, we find that achieving accurate
risk distribution matching using distribution samples requires larger
batch sizes - details of which are examined in our ablation studies.
For most datasets, the optimal batch size lies between $\left[70,100\right]$.
However, for huge datasets like TerraIncognita and DomainNet, it is
between $\left[30,60\right]$. Although computational resources limit
us from testing larger batch sizes, these ranges consistently achieve
strong performance on benchmarks. The matching coefficient $\lambda$
in our method is set in $[0.1,10.0]$. Additional hyper-parameters
like learning rate, dropout rate, or weight decay, adhere to the preset
ranges as detailed in~\cite{eastwood2022probable}. We provide more
implementation details in the supplementary material. We repeat our
experiments ten times with varied seed values and hyper-parameters
and report the average results.

\paragraph*{Experimental Results}

In Table~\ref{tab:DomainBed-results-mainpaper}, we show the average
out-of-domain (OOD) accuracies of state-of-the-art DG methods on five
benchmarks. Due to space constraints, domain-specific accuracies are
detailed in the supplementary material. We compare $\Model$ with
ERM and various types of algorithms: distributional robustness (GroupDRO),
causal methods learning invariance (IRM, VREx, EQRM), gradient matching
(Fish~\cite{shi2022gradient}, Fishr~\cite{rame2022fishr}), representation
distribution matching (MMD, CORAL) and other variants (Mixup~\cite{zhang2017mixup},
MLDG~\cite{li2018learning}). To ensure fairness in our evaluations,
we have used the same training data volume across all baselines, although
further employing augmentations can enhance models' performance.

On average, $\Model$ surpasses other baselines across all benchmarks,
notably achieving a $1.5$\% average improvement over $\ModelERM$.
The significant improvement of $\Model$ on DomainNet, a large-scale
dataset with 586,575 images across 6 domains, is worth mentioning.
This suggests that characterising domains with risk distributions
to achieve invariance effectively enhances OOD performance. Compared
to distributional robustness methods, $\Model$ notably outperforms
GroupDRO with improvement of $2.8$\% on PACS and a substantial $10.1$\%
on DomainNet. $\Model$ consistently improves over causality-based
methods that rely on the average risk for domain invariance. This
superiority attributes to our novel adoption of risk distributions,
achieving enhanced invariance for DG. Our remarkable improvement over
$\MMD$ suggests that aligning \emph{risk distributions} via the MMD
distance is more effective, easier to optimise than aligning representation
distributions. While $\Model$ typically outperforms $\CORAL$ and
Fish in OOD scenarios, it only remains competitive or sometimes underperforms
on certain datasets like OfficeHome. This decrease in performance
may stem from the dataset's inherent tendency to overfit within our
risk distribution alignment objective. OfficeHome has only average
about 240 samples per class, significantly fewer than other datasets
with at least 1,400. This reduced sample size may not provide sufficiently
diverse risk distributions to capture stable class features, resulting
in overfitting on the training set. Despite these limitations, our
OfficeHome results still outperform several well-known baselines such
as MLDG, VREx, or ERM. For a detailed discussion on this challenge,
please refer to our supplementary material.

\subsection{Analysis}

\begin{figure}
\subfloat[A small gap between two variants of $\protect\Model$.\label{fig:analysis_A-small-gap}]{\begin{centering}
\includegraphics[width=0.455\columnwidth]{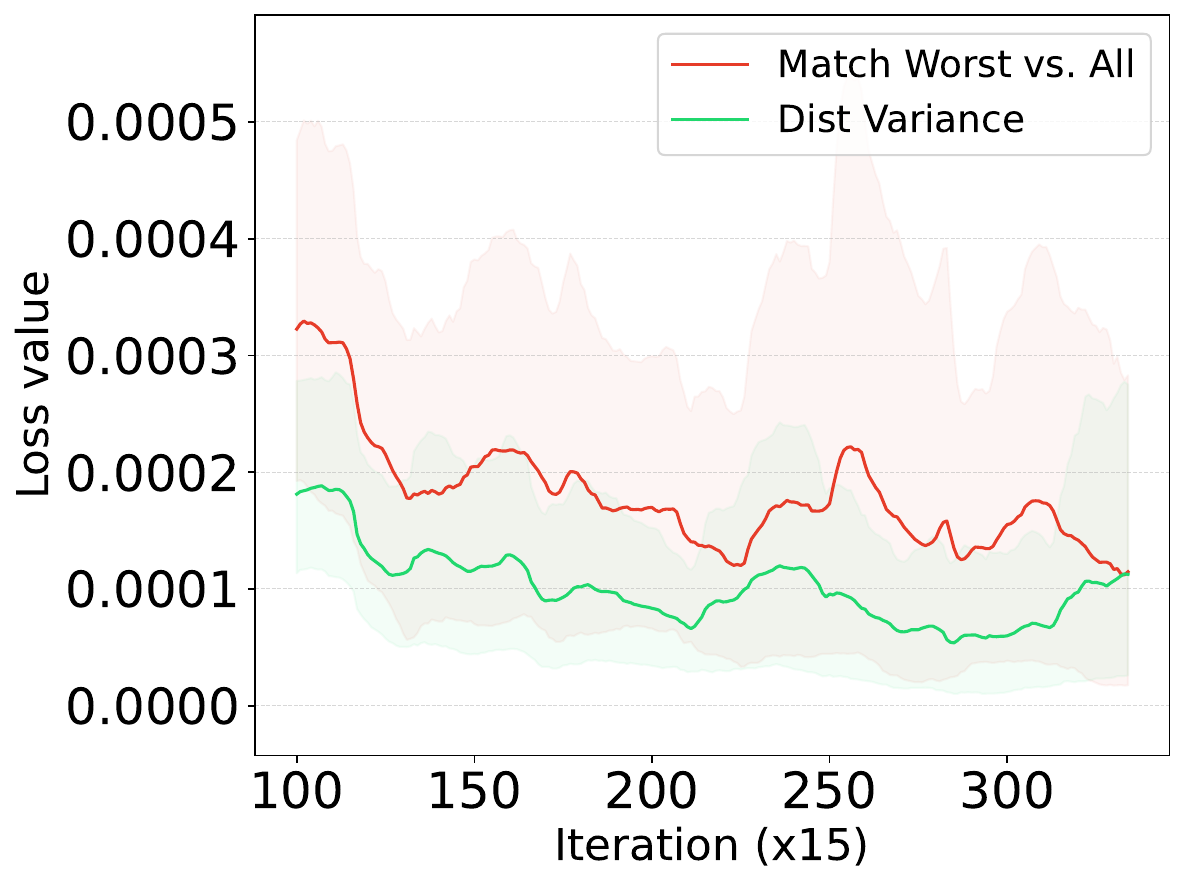}
\par\end{centering}

}\hfill{}\subfloat[Learning curves of $\protect\Model$ and other baselines.\label{fig:analysis_Learning-curves}]{\begin{centering}
\includegraphics[width=0.465\columnwidth]{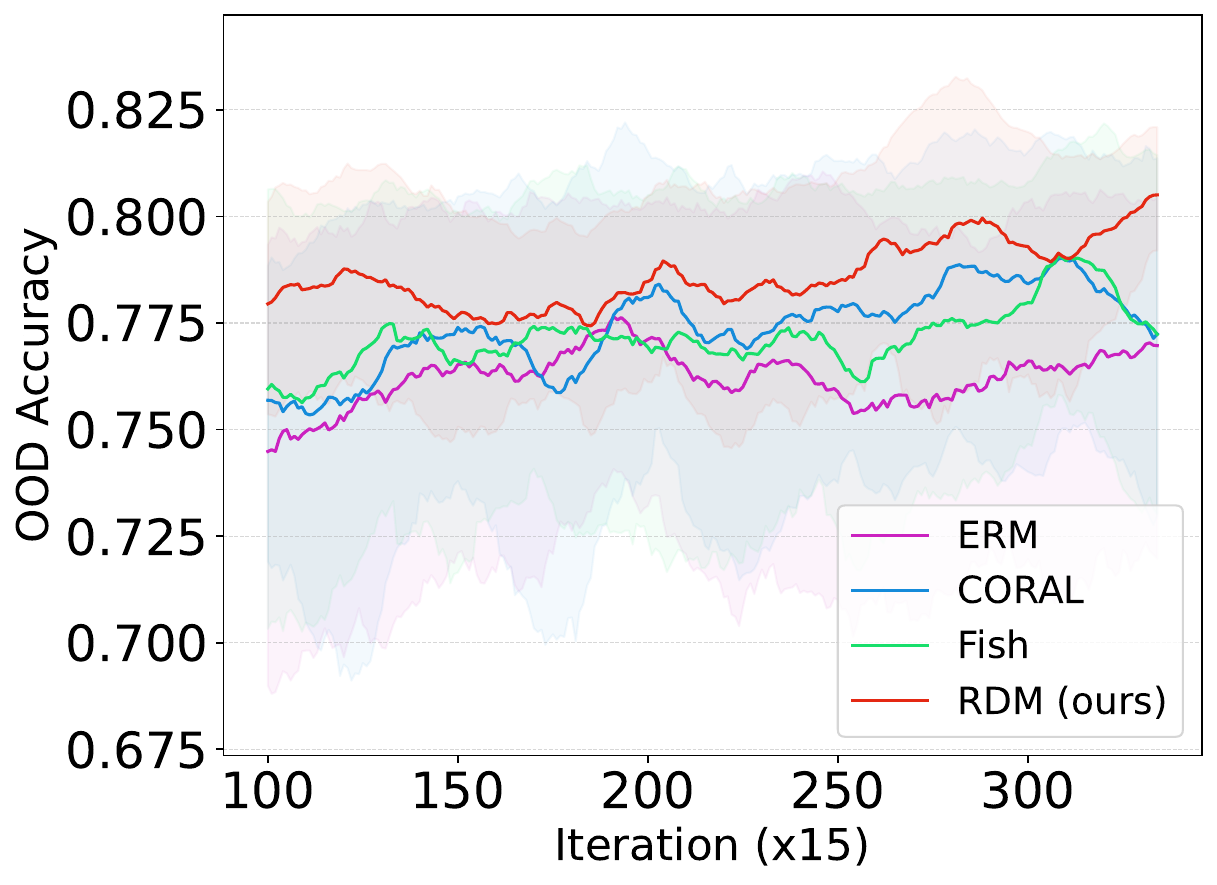}
\par\end{centering}

}

\caption{Figure~\ref{fig:analysis_A-small-gap} supports our claim about the
approximation of the distributional variance, while Figure~\ref{fig:analysis_Learning-curves}
compares the OOD performance learning curves of RDM with other methods.
These insights are visualised every 15 iterations during PACS dataset
training, excluding the OOD Sketch domain. After $\protect\Model$
is pre-trained with ERM for $100(\times15)$ iterations, our visual
analysis commences, ensuring a fair comparison.}
\end{figure}

In this section, we provide empirical evidence backing our claims
in Section~\ref{sec:Methodology}. In Figure~\ref{fig:analysis_A-small-gap},
we highlight a small gap when aligning the risk distribution of the
worst-case domain with that of all domains combined ($\Model$ with
$\mathcal{\hat{L}}_{\Model}$), compared to directly optimising the
distributional variance ($\Model$ with $\mathcal{\mathcal{L}}_{\Model}$).
Notably, $\mathcal{\hat{L}}_{\Model}$ consistently represents an
upper bound of $\mathcal{\mathcal{L}}_{\Model}$, which is sensible
since the worst-case domain often exhibits the most distinct risk
distribution. This suggests that optimising $\mathcal{\hat{L}}_{\Model}$
also helps reduce the distributional variance $\mathcal{\mathcal{L}}_{\Model}$,
bringing the risk distributions across domains closer.

\begin{table}
\begin{centering}
\begin{tabular}{cccc}
\toprule 
\textbf{Algorithm} & Training (s) & Mem (GiB) & Acc (\%)\tabularnewline
\midrule
\midrule 
Fish & 11,502 & \textbf{5.26} & 42.7\tabularnewline
\midrule 
CORAL & 11,504 & 17.00 & 41.5\tabularnewline
\midrule 
$\Model$ with $\mathcal{\mathcal{L}}_{\Model}$ & 9,854 & 16.94 & 43.1\tabularnewline
\midrule 
$\Model$ with $\mathcal{\hat{L}}_{\Model}$ & \textbf{7,749} & 16.23 & \textbf{43.4}\tabularnewline
\bottomrule
\end{tabular}
\par\end{centering}
\caption{Comparison between Fish, CORAL, and two variants of our method in
terms of the training time (seconds), memory usage per iteration (GiB)
and accuracy (\%) on DomainNet. \label{tab:Efficiency-and-Effectiveness-compare}}
\end{table}

When the number of training domains grows, especially with large-scale
datasets like DomainNet, emphasising the risk distribution of the
worst-case domain not only proves to be a more efficient approach
but also significantly enhances OOD performance. In our exploration
of training resources for DomainNet, we study three matching methods:
Fish, $\CORAL$ and two variants of our $\Model$ method. For a fair
evaluation, all experiments were conducted with identical GPU resources,
settings, and hyper-parameters, such as batch size or training iterations.
Results can be seen in Table~\ref{tab:Efficiency-and-Effectiveness-compare}.
Full details on training resources for these methods on other datasets
are available in the supplementary material due to space constraints. 

Our $\Model$ with the $\mathcal{\hat{L}}_{\Model}$ objective proves
fastest in training and achieves the notably highest $43.4$\% accuracy
on DomainNet. While $\Model$ demands more memory than Fish, due to
the storage of $\MMD$ distance values, it can be trained in less
time - under an hour - and still delivers a $0.7$\% performance boost.
This gain over Fish, a leading gradient maching method on DomainNet,
is significant. Among two variants of $\Model$, the one using $\mathcal{\hat{L}}_{\Model}$
is both the fastest and most accurate, justifying our claims on the
benefits of aligning the risk distribution of the worst-case domain. 

More, to further highlight the efficacy of risk distribution alignment
for DG, we compare the OOD performance learning curves of $\Model$
with competing baselines using representation (CORAL) and gradient
(Fish) alignments, as depicted in Figure~\ref{fig:analysis_Learning-curves}.
Impressively, $\Model$ consistently outperforms, demonstrating enhanced
generalisation throughout the training process.

\subsection{Ablation studies}

We explore the impact of the matching coefficient $\lambda$ and training
batch size on risk distribution matching, using primarily the PACS
dataset for brevity. While other datasets exhibit similar trends,
their detailed results are provided in the supplementary material.

\paragraph*{Matching coefficient $\lambda$}

Figure~\ref{fig:ablation_study_lambda} illustrates the performance
of $\Model$ on the PACS dataset for varying values of the matching
coefficient $\lambda$, spanning $\left\{ 0.1,1.0,2.5,5.0,7.5,10.0\right\} $.
Notably, as $\lambda$ increases, $\Model$'s accuracy consistently
improves, justifying the significance of our risk distribution matching
module in fostering generalisation. In particular, when $\lambda=5.0$,
$\Model$ demonstrates a notable $1.6$\% average accuracy boost across
all domains, in contrast to when using only $\lambda=0.1$. Across
most datasets, a $\lambda$ value within $\left[0.1,10.0\right]$
appears sufficient to produce good results.

\paragraph*{Batch size }

We study the impact of batch size on $\Model$'s performance. Our
assumption is that achieving accurate risk distribution matching through
data samples would require larger batch sizes. Figure~\ref{fig:ablation_study_batchsize}
validates this, revealing enhanced generalisation results on PACS
with increased batch sizes. For PACS, sizes between $[70,100]$ yield
promising, potentially optimal outcomes, despite computational limitations
restrict our exploration of larger sizes.

\begin{figure}
\subfloat[Matching coefficient $\lambda$.\label{fig:ablation_study_lambda}]{\begin{centering}
\includegraphics[width=0.49\columnwidth]{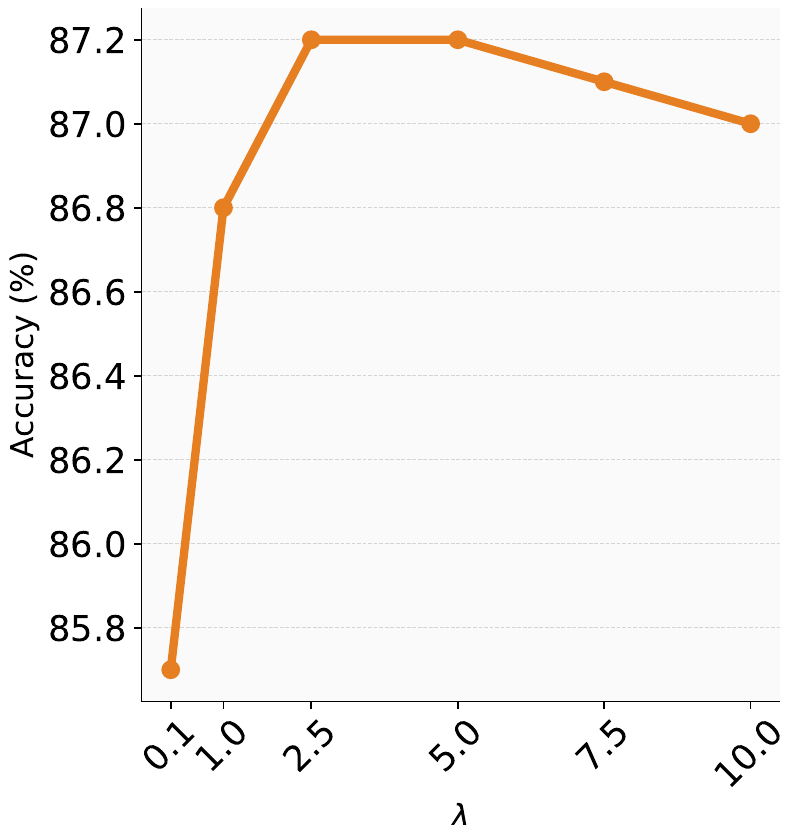} 
\par\end{centering}
}\hspace*{\fill}\subfloat[Batch size.\label{fig:ablation_study_batchsize}]{\begin{centering}
\includegraphics[width=0.49\columnwidth]{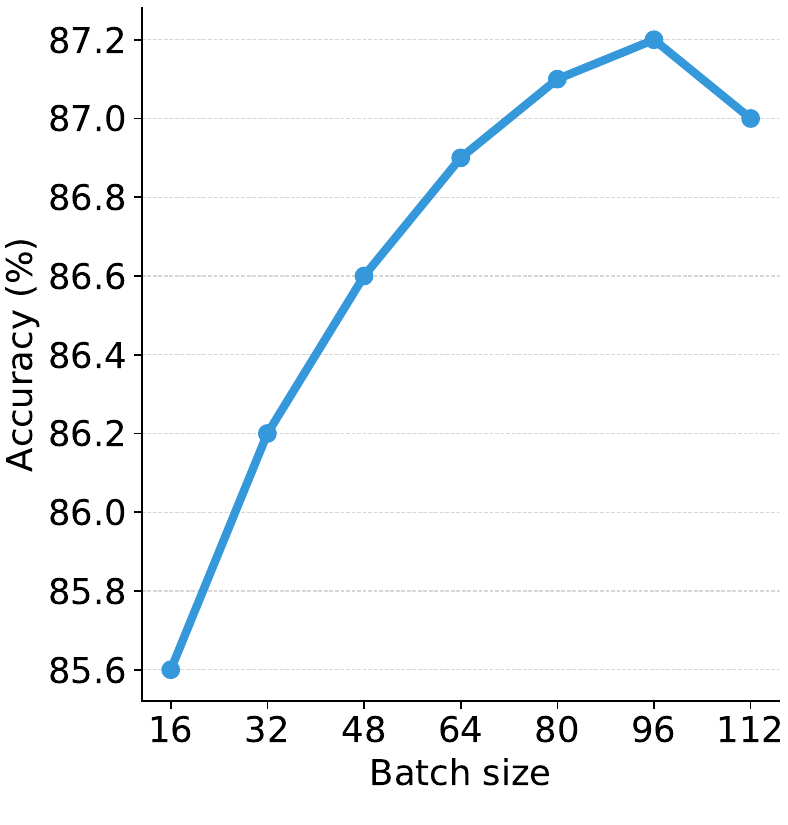} 
\par\end{centering}
}

\caption{Ablation studies on the effects of the matching coefficient $\lambda$
and the training batch size on the PACS dataset. }
\end{figure}

\section{Conclusion\label{sec:Conclusion}}

We have demonstrated that $\Model$, a novel matching method for domain
generalisation (DG), provides enhanced generalisation capability by
aligning risk distributions across domains. $\Model$ efficiently
overcomes high-dimensional challenges of conventional DG matching
methods. $\Model$ is built on our observation that risk distributions
can effectively represent the differences between training domains.
By minimising these divergences, we can achieve an invariant and generalisable
predictor. We further improve $\Model$ by matching only the risk
distribution of the worst-case domain with the aggregate from all
domains, bypassing the need to directly compute the distributional
variance. This approximate version not only offers computational efficiency
but also delivers improved out-of-domain results. Our extensive experiments
on several benchmarks reveal that $\Model$ surpasses leading DG techniques.
We hope our work can inspire further investigations into the benefits
of risk distributions for DG. 
{\small{}\bibliographystyle{ieee_fullname}
\bibliography{wacv24}
 }{\small\par}

\pagebreak{}

\section{Supplementary Material}

In this supplementary material, we first provide a detailed proof
for our theorem on distributional variance, as outlined in Section~\ref{sec:supp-Theoretical-Results}.
Next, in Section~\ref{sec:supp-More-Implementation-Details}, we
detail more about our experimental settings, covering both the $\ColoredMNIST$
synthetic dataset~\cite{arjovsky2019invariant} and the extensive
benchmarks from the DomainBed suite~\cite{gulrajani2021in} in the
main text. Additional ablation studies and discussions on our proposed
method are given in Section~\ref{sec:supp-Additional-Ablation-Studies}.
Finally, Section~\ref{sec:supp-More-experimental-results} provides
domain-specific out-of-domain accuracies for each dataset within the
DomainBed suite.

\section{Theoretical Results\label{sec:supp-Theoretical-Results}}

We provide the proof for the theorem on distributional variance discussed
in the main paper. We revisit the concept of kernel mean embedding~\cite{smola2007hilbert}
to express the risk distribution $\mathcal{T}_{e}$ of domain $e$.
Particularly, we represent $\mathcal{T}_{e}$ through its embedding,
$\mu_{\mathcal{T}_{e}}$, in a reproducing kernel Hilbert space (RKHS)
denoted as $\mathcal{H}$. This is achieved by using a feature map
$\phi:\mathbb{R}\rightarrow\mathcal{H}$ below: 
\begin{align}
\mu_{\mathcal{T}_{e}} & \coloneqq\mathbb{E}_{R_{e}\sim\mathcal{T}_{e}}\left[\phi\left(R_{e}\right)\right]\label{eq:mean_embed-1}\\
 & =\ \mathbb{E}_{R_{e}\sim\mathcal{T}_{e}}\left[k\left(R_{e},\cdot\right)\right]
\end{align}
where a kernel function $k\left(\cdot,\cdot\right):\mathbb{R}\times\mathbb{R}\rightarrow\mathbb{R}$
is introduced to bypass the explicit specification of $\phi$.
\begin{thm*}
~\cite{muandet2013domain} Denote $\mathcal{T}=\frac{1}{m}\sum_{e=1}^{m}\mathcal{T}_{e}$
the probability distribution over the risks of all samples in the
entire training set, or equivalently, the set of all $m$ domains.
Given the distributional variance $\mathbb{V}_{\mathbb{\mathcal{H}}}\left(\left\{ \mathcal{T}_{1},...,\mathcal{T}_{m}\right\} \right)$
is calculated with a characteristic kernel $k$, $\mathbb{V}_{\mathcal{H}}\left(\left\{ \mathcal{T}_{1},...,\mathcal{T}_{m}\right\} \right)=0$
if and only if $\mathcal{T}_{1}=...=\mathcal{T}_{m}\left(=\mathcal{T}\right)$.
\end{thm*}
\begin{proof}
In our methodology, we employ the RBF kernel, which is \emph{characteristic}
in nature. As a result, the term $\left\Vert \mu_{\mathcal{T}_{e}}-\mu_{\mathcal{T}}\right\Vert _{\mathcal{H}}^{2}$
acts as a metric within the Hilbert space $\mathcal{H}$~\cite{muandet2013domain}.
Importantly, this metric reaches zero if and only if $\left(\mathcal{T}_{e}=\mathcal{T}\right)$~\cite{sriperumbudur2010hilbert}.
Let's consider the distributional variance, $\mathbb{V}_{\mathcal{H}}\left(\left\{ \mathcal{T}_{1},...,\mathcal{T}_{m}\right\} \right)$,
which is defined below:

\begin{equation}
\mathbb{V}_{\mathcal{H}}=\frac{1}{m}\sum_{e=1}^{m}\left\Vert \mu_{\mathcal{T}_{e}}-\mu_{\mathcal{T}}\right\Vert _{\mathcal{H}}^{2}
\end{equation}
This variance becomes zero if and only if $\left\Vert \mu_{\mathcal{T}_{e}}-\mu_{\mathcal{T}}\right\Vert _{\mathcal{H}}^{2}=0$
for each $e$. This logically implies that $\left(\mathcal{T}_{e}=\mathcal{T}\right)$
for all $e$, leading to $\left(\mathcal{T}_{1}=\mathcal{T}_{2}=...=\mathcal{T}_{m}\right)$. 

Conversely, we assume that $\left(\mathcal{T}_{1}=\mathcal{T}_{2}=...=\mathcal{T}_{m}\right)$.
Given this condition, for any $e$, it follows that:

\begin{equation}
\mu_{\mathcal{T}}=\frac{1}{m}\sum_{e=1}^{m}\mu_{\mathcal{T}_{e}}=\mu_{\mathcal{T}_{e}}
\end{equation}
which implies

\begin{equation}
\left\Vert \mu_{\mathcal{T}_{e}}-\mu_{\mathcal{T}}\right\Vert _{\mathcal{H}}^{2}=0.
\end{equation}
Consequently, by the given definition of distributional variance,
we have: $\mathbb{V}_{\mathcal{H}}\left(\left\{ \mathcal{T}_{1},...,\mathcal{T}_{m}\right\} \right)=\frac{1}{m}\sum_{e=1}^{m}\left\Vert \mu_{\mathcal{T}_{e}}-\mu_{\mathcal{T}}\right\Vert _{\mathcal{H}}^{2}=0$.
This completes the proof.
\end{proof}

\section{More implementation details\label{sec:supp-More-Implementation-Details}}

For our experiments, we leveraged the PyTorch DomainBed toolbox~\cite{gulrajani2021in,eastwood2022probable}
and utilised an Ubuntu 20.4 server outfitted with a 36-core CPU, 767GB
RAM, and NVIDIA V100 32GB GPUs. The software stack included Python
3.11.2, PyTorch 1.7.1, Torchvision 0.8.2, and Cuda 12.0. Additional
implementation details, beyond the hyper-parameters discussed in the
main text, are elaborated below.

\subsection{ColoredMNIST}

In alignment with~\cite{eastwood2022probable}, we performed experiments
on the ColoredMNIST dataset, the results of which are detailed in
Table~\ref{tab:-CMNIST_result_mainpaper} in the main paper. We partitioned
the original MNIST training dataset into distinct training and validation
sets of 25,000 and 5,000 samples for each of two training domains,
respectively. The original MNIST test set was adapted to function
as our test set. Particularly, we synthesised this test set to introduce
a distribution shift: red digits have only a $10$\% probability of
being classified as ``zero'', compared to 80\% and 90\% in the training
sets for different domains. Besides the hyper-parameters highlighted
in the main paper, we also leveraged a cosine annealing scheduler
to further optimise the training process like other baselines. 

For our $\Model$ method, we constrained the alignment to focus only
on \emph{the first two empirical moments (mean and variance)} of $\mathcal{T}_{w}$
and $\mathcal{T}$. We experimented with five different penalty weight
values for $\lambda$ in the range of $\left\{ 500,1000,2500,5000,10000\right\} $,
running each experiment ten times and varying $\lambda$. The reported
results are the average accuracies and their standard deviations over
these 10 runs, all measured on a test-domain test set. We adhered
to test-domain validation for model selection across all methods,
as recommended by~\cite{gulrajani2021in}. We reference results for
other methods from~\cite{eastwood2022probable}.

\subsection{DomainBed}

\subsubsection{Description of benchmarks}

\begin{table*}
\begin{centering}
\begin{tabular}{>{\raggedright}m{0.25\textwidth}>{\centering}p{0.3\textwidth}cc}
\toprule 
Parameter & Dataset & Default value & Random distribution\tabularnewline
\midrule
\midrule 
steps & All & 5,000 & 5,000\tabularnewline
\midrule 
learning rate & All & $5\text{e-5}$ & $10^{\text{Uniform\ensuremath{\left(-4.5,-4\right)}}}$\tabularnewline
\midrule 
dropout & All & 0 & RandomChoice$\left(\left[0,0.003,0.03\right]\right)$\tabularnewline
\midrule 
weight decay & All & 0 & $10^{\text{Uniform\ensuremath{\left(-8,-5\right)}}}$\tabularnewline
\midrule 
\multirow{2}{0.25\textwidth}{batch size} & PACS / VLCS / OfficeHome & 88 & Uniform($70,100$)\tabularnewline
\cmidrule{2-4} \cmidrule{3-4} \cmidrule{4-4} 
 & TerraIncognita / DomainNet & 40 & Uniform($30,60$)\tabularnewline
\midrule 
\multirow{2}{0.25\textwidth}{matching coefficient $\lambda$} & All except DomainNet & 5.0 & Uniform($0.1,10.0$)\tabularnewline
\cmidrule{2-4} \cmidrule{3-4} \cmidrule{4-4} 
 & DomainNet & 0.5 & Uniform($0.1,1.0$)\tabularnewline
\midrule 
\multirow{2}{0.25\textwidth}{pre-trained iterations} & All except DomainNet & 1500 & Uniform($800,2700$)\tabularnewline
\cmidrule{2-4} \cmidrule{3-4} \cmidrule{4-4} 
 & DomainNet & 2400 & Uniform($1500,3000$)\tabularnewline
\midrule 
learning rate after pre-training & All & $1.5\text{e-5}$ & Uniform($8\text{e-6},2\text{e-5}$)\tabularnewline
\midrule 
\multirow{2}{0.25\textwidth}{variance regularisation \\
coefficient} & PACS / VLCS & 0.004 & Uniform($0.001,0.007$)\tabularnewline
\cmidrule{2-4} \cmidrule{3-4} \cmidrule{4-4} 
 & OfficeHome / TerraIncognita /\\
DomainNet & 0 & 0\tabularnewline
\bottomrule
\end{tabular}
\par\end{centering}
\caption{Hyper-parameters, along with their default values and distributions,
are optimised through random search across the five benchmark datasets.\label{tab:supp-Hyperparameters}}
\end{table*}

For our evaluations, we leveraged five large-scale benchmark datasets
from the DomainBed suite~\cite{gulrajani2021in}, comprising:
\begin{itemize}
\item VLCS~\cite{fang2013unbiased}: The dataset encompasses four photographic
domains: Caltech101, LabelMe, SUN09, VOC2007. It contains 10,729 examples,
each with dimensions $\left(3,224,224\right)$, and spans five distinct
classes.
\item PACS~\cite{li2017deeper}: The dataset includes 9,991 images from
four different domains: Photo (P), Art-painting (A), Cartoon (C),
and Sketch (S). These domains each have their own unique style, making
this dataset particularly challenging for out-of-distribution (OOD)
generalisation. Each domain has seven classes.
\item OfficeHome~\cite{venkateswara2017deep}: The dataset features 15,500
images of objects commonly found in office and home settings, categorised
into 65 classes. These images are sourced from four distinct domains:
Art (A), Clipart (C), Product (P), and Real-world (R).
\item TerraIncognita~\cite{beery2018recognition}: The dataset includes
24,788 camera-trap photographs of wild animals captured at locations
$\left\{ \text{L100},\text{L38},\text{L43},\text{L46}\right\} $.
Each image has dimensions $\left(3,224,224\right)$ and falls into
one of 10 distinct classes.
\item DomainNet~\cite{peng2019moment}: The largest dataset in DomainBed,
DomainNet, contains 586,575 examples in dimensions $\left(3,224,224\right)$,
spread across six domains $\left\{ \text{clipart},\text{infograph},\text{painting},\text{quickdraw},\text{real},\text{sketch}\right\} $
and encompassing 345 classes.
\end{itemize}

\subsubsection{Our implementation details}

\begin{table*}[t]
\subfloat[PACS]{\begin{centering}
\begin{tabular}{cccc}
\toprule 
\textbf{Algorithm} & Training (s) & Mem (GiB) & Acc (\%)\tabularnewline
\midrule
\midrule 
Fish & 7,566 & \textbf{7.97} & 85.5\tabularnewline
\midrule 
CORAL & 4,485 & 21.81 & 86.2\tabularnewline
\midrule 
$\Model$ with $\mathcal{\mathcal{L}}_{\Model}$ & 4,783 & 21.87 & 86.6\tabularnewline
\midrule 
$\Model$ with $\mathcal{\hat{L}}_{\Model}$ & \textbf{4,214} & 21.71 & \textbf{87.2}\tabularnewline
\bottomrule
\end{tabular}
\par\end{centering}
}\hspace*{\fill}\subfloat[VLCS]{\begin{centering}
\begin{tabular}{cccc}
\toprule 
\textbf{Algorithm} & Training (s) & Mem (GiB) & Acc (\%)\tabularnewline
\midrule
\midrule 
Fish & 13,493 & \textbf{7.97} & 77.8\tabularnewline
\midrule 
CORAL & 6,329 & 21.81 & \textbf{78.8}\tabularnewline
\midrule 
$\Model$ with $\mathcal{\mathcal{L}}_{\Model}$ & 9,441 & 21.87 & 77.8\tabularnewline
\midrule 
$\Model$ with $\mathcal{\hat{L}}_{\Model}$ & \textbf{6,151} & 21.71 & 78.4\tabularnewline
\bottomrule
\end{tabular}
\par\end{centering}
}

\subfloat[OfficeHome]{\begin{centering}
\begin{tabular}{cccc}
\toprule 
\textbf{Algorithm} & Training (s) & Mem (GiB) & Acc (\%)\tabularnewline
\midrule
\midrule 
Fish & 9,035 & \textbf{7.97} & 68.6\tabularnewline
\midrule 
CORAL & 4,762 & 21.81 & \textbf{68.7}\tabularnewline
\midrule 
$\Model$ with $\mathcal{\mathcal{L}}_{\Model}$ & 5,467 & 21.87 & 67.0\tabularnewline
\midrule 
$\Model$ with $\mathcal{\hat{L}}_{\Model}$ & \textbf{4,588} & 21.71 & 67.3\tabularnewline
\bottomrule
\end{tabular}
\par\end{centering}
}\hspace*{\fill}\subfloat[TerraIncognita]{\begin{centering}
\begin{tabular}{cccc}
\toprule 
\textbf{Algorithm} & Training (s) & Mem (GiB) & Acc (\%)\tabularnewline
\midrule
\midrule 
Fish & 6,019 & \textbf{4.08} & 45.1\tabularnewline
\midrule 
CORAL & 2,973 & 10.21 & \textbf{47.6}\tabularnewline
\midrule 
$\Model$ with $\mathcal{\mathcal{L}}_{\Model}$ & 4,040 & 10.17 & 47.1\tabularnewline
\midrule 
$\Model$ with $\mathcal{\hat{L}}_{\Model}$ & \textbf{2,697} & 10.11 & 47.5\tabularnewline
\bottomrule
\end{tabular}
\par\end{centering}
}

\caption{Comparison between Fish, CORAL, and two variants of our method in
terms of the training time (seconds), memory usage per iteration (GiB)
and accuracy (\%) on PACS, VLCS, OfficeHome and TerraIncognita. \label{tab:supp-Efficiency-and-Effectiveness-compare}}
\end{table*}

To ensure rigorous evaluation and a fair comparison with existing
baselines~\cite{eastwood2022probable,rame2022fishr}, we conducted
experiments on five datasets from the DomainBed suite, the results
of which are elaborated in Table~\ref{tab:DomainBed-results-mainpaper}
of the main text. In alignment with standard practices, we optimised
hyper-parameters for each domain through a randomised search across
20 trials on the validation set, utilising a joint distribution as
specified in Table~\ref{tab:supp-Hyperparameters}. The dataset from
each domain was partitioned into an 80\% split for training and testing,
and a 20\% split for hyper-parameter validation. A comprehensive discussion
on the hyper-parameters used in our experiments is provided below.
For each domain, we performed our experiments ten times, employing
varied seed values and hyper-parameters within the specified range,
and reported the averaged results with their standard deviations.
We reference results for other methods from~\cite{eastwood2022probable,rame2022fishr}.
We kindly refer readers to our given source code for more detail.

In our methodology, we employed the MMD distance for aligning risk
distributions $\mathcal{T}_{e}$ and $\mathcal{T}$, as described
in Section~\ref{sec:Methodology}. Utilising the RBF kernel, we compute
the average MMD distance across an expansive bandwidth spectrum $\left\{ 0.0001,0.001,0.01,0.1,1,10,100,1000\right\} $,
bypassing the need for tuning this parameter.

Inspired by recent insights~\cite{eastwood2022probable}, we incorporated
an initial pre-training-with-ERM phase to further improve the OOD
performance. DomainNet, given its scale, requires longer ERM pre-training;
specific parameters for all datasets are provided in Table~\ref{tab:supp-Hyperparameters}.
Our initial learning rate lies within $\left[1\text{e-4.5},1\text{e-4}\right]$,
which adapts to $\left[8\text{e-6},2\text{e-5}\right]$ post-ERM pre-training.
Incorporating additional variance regularisation on $\mathcal{T}_{e}$
and $\mathcal{T}$ proves beneficial for the PACS and VLCS datasets.
This approach constrains the induced risks to fall within narrower,
more optimal value ranges, facilitating more effective risk distribution
alignment. Optimal regularisation coefficients for this strategy are
detailed in Table~\ref{tab:supp-Hyperparameters}.

We maintain minimal dropout and weight decay, reserving our focus
for risk distribution alignment. Optimal batch sizes differ: $\left[70,100\right]$
for VLCS and OfficeHome, and $\left[30,60\right]$ for TerraIncognita
and DomainNet. Despite computational constraints limiting our ability
to test larger batch sizes, the selected ranges yield robust performance
across datasets.

Regarding the matching coefficient $\lambda$ in our objective, most
datasets work well within $\left[0.1,10.0\right]$, but DomainNet
prefers a narrower $\left[0.1,1.0\right]$ range. This fine-tuning
is key, especially for large-scale datasets, to balance risk reduction
and cross-domain alignment in the early training stages.

\section{More ablation studies and analyses\label{sec:supp-Additional-Ablation-Studies}}

\subsection{Efficacy of DG matching methods }

Table~\ref{tab:supp-Efficiency-and-Effectiveness-compare} compares
the efficiency and effectiveness of various methods: Fish, CORAL,
$\Model$ with $\mathcal{\mathcal{L}}_{\Model}$, and $\Model$ with
$\mathcal{\hat{L}}_{\Model}$ across several benchmarks - PACS, VLCS,
OfficeHome, and TerraIncognita. Notably, the approximate variant,
denoted as $\Model$ with $\mathcal{\hat{L}}_{\Model}$, stands out
for its exceptional performance. This version emphasises the alignment
of risk distribution for the worst-case domain and exhibits both faster
training times and improved accuracy over its counterpart that optimises
distributional variance, $\Model$ with $\mathcal{\mathcal{L}}_{\Model}$.
For instance, on the VLCS dataset, this variant is trained in under
an hour while achieving a $0.6$\% accuracy boost.

When compared to the gradient-matching Fish method, our approach demonstrates
similar advantages but requires additional memory to store MMD distance
values. The memory constraint is not unique to our method; CORAL also
encounters this limitation. However, $\Model$ outperforms CORAL in
both training time and memory usage, especially evident on large-scale
datasets like DomainNet. This efficiency gain is noteworthy, given
that CORAL's increased computational requirements arise from its handling
of high-dimensional representation vectors.

In terms of the accuracy, as confirmed by our main text, $\Model$
outperforms CORAL substantially on both PACS and DomainNet, while
maintaining competitiveness on TerraIncognita and VLCS. On OfficeHome,
although $\Model$ lags behind CORAL, we provide an in-depth explanation
for this behavior both in the main text and in the subsequent section.

\subsection{Decreased Performance on OfficeHome}

\begin{figure}[h]
\begin{centering}
\includegraphics[width=0.99\columnwidth]{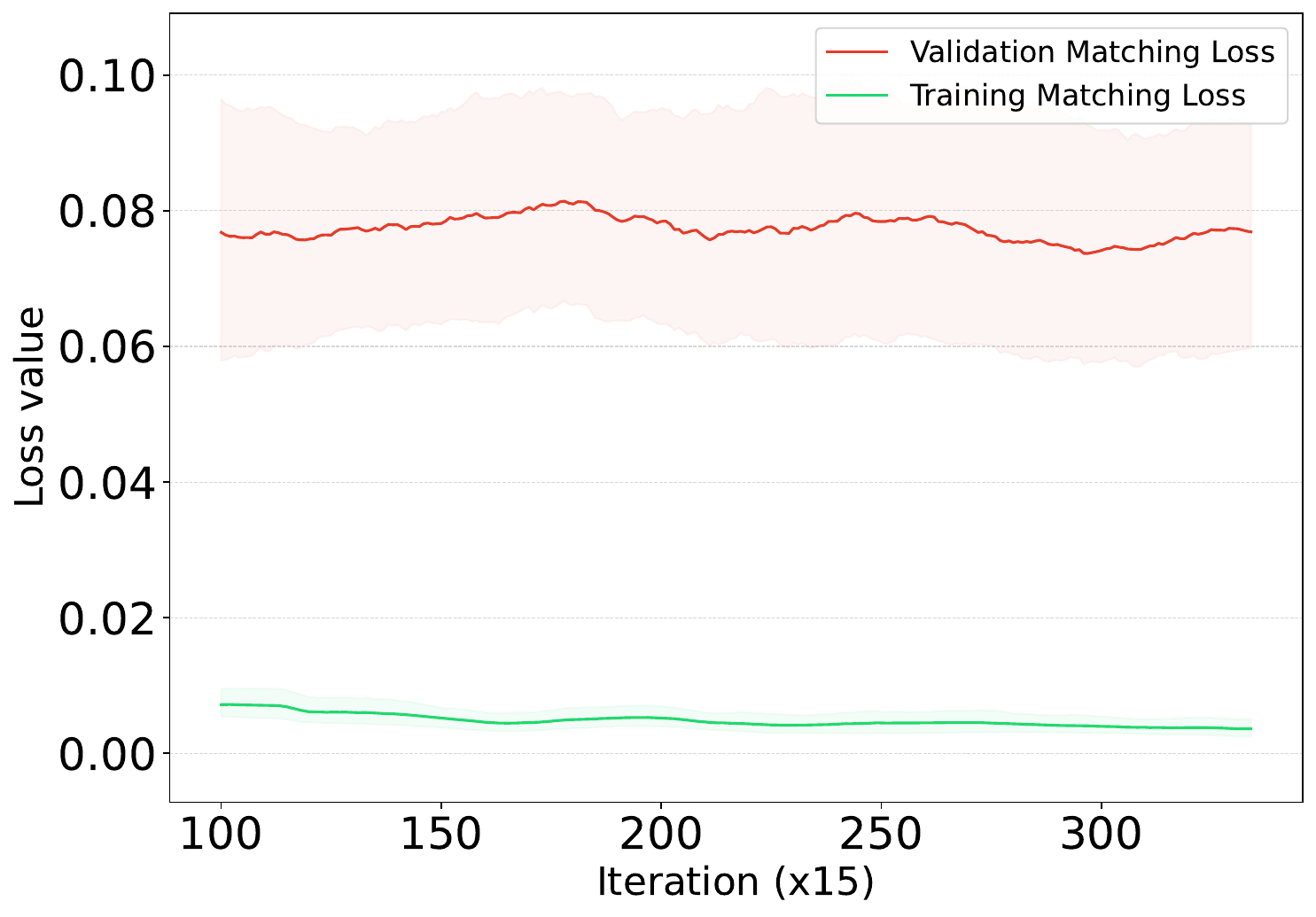}
\par\end{centering}
\caption{The notable gap between training and validation matching loss on the
OfficeHome dataset, excluding the OOD Art domain. Analysis begins
after $\protect\Model$ completes 1,500 pre-training iterations via
ERM. Metrics recorded at every 15-iteration interval.\label{fig:supp-A-large-gap-OfficeHome}}

\end{figure}

In our evaluation, $\Model$ generally surpasses competing matching
methods in OOD settings but faces challenges in specific datasets
like OfficeHome. The dataset's limitations are noteworthy: with an
average of only 240 samples per class, OfficeHome has significantly
fewer instances per class than other datasets, which usually have
at least 1,400. This limited sample size may constrain the model from
learning sufficiently class-semantic features or diverse risk distributions,
leading to overfitting on the training set. To shed light on this
issue, we present a visual analysis in Figure~\ref{fig:supp-A-large-gap-OfficeHome}.
Starting from the 100th iteration, when we perform the task of matching
risk distributions, we note that the training matching loss is already
minimal, forming a clear divergence with the validation matching loss.
While the training loss continues to converge to minimal values, the
validation loss remains inconsistent throughout the training phase.
This inconsistency showcases that the limited diversity in OfficeHome's
risk distributions may induce the model's overfitting on training
samples, reducing its generalisation capabilities. Despite these constraints,
our method still outperforms other well-known baselines, such as MLDG,
VREx, and ERM, on OfficeHome.

\subsection{Impact of batch size and matching coefficient}

\begin{figure*}
\subfloat[VLCS]{\begin{centering}
\includegraphics[width=0.23\textwidth]{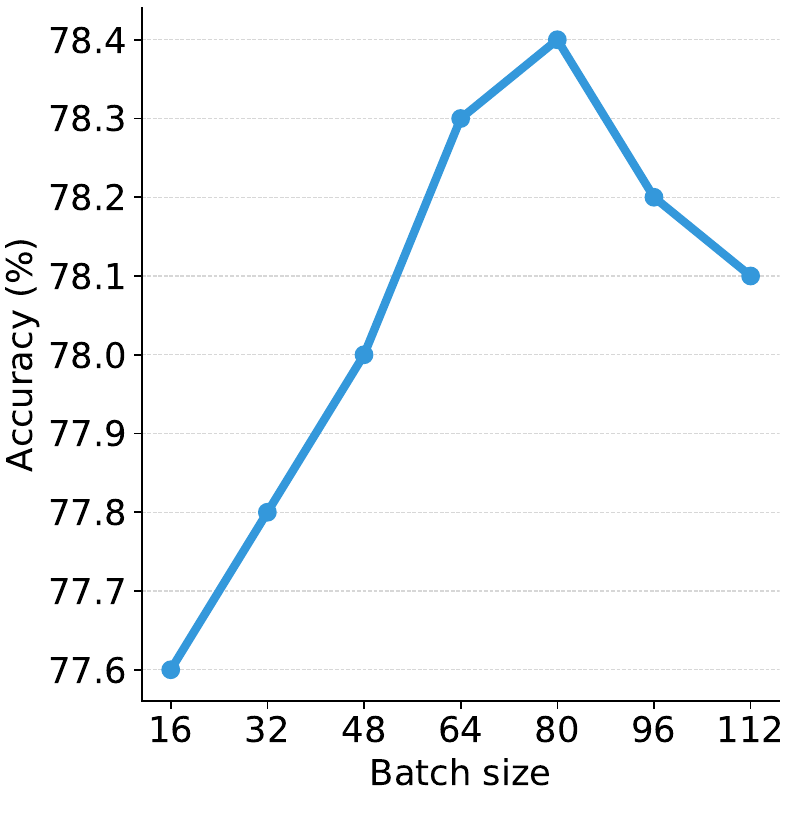}
\par\end{centering}
}\hfill{}\subfloat[OfficeHome]{\begin{centering}
\includegraphics[width=0.23\textwidth]{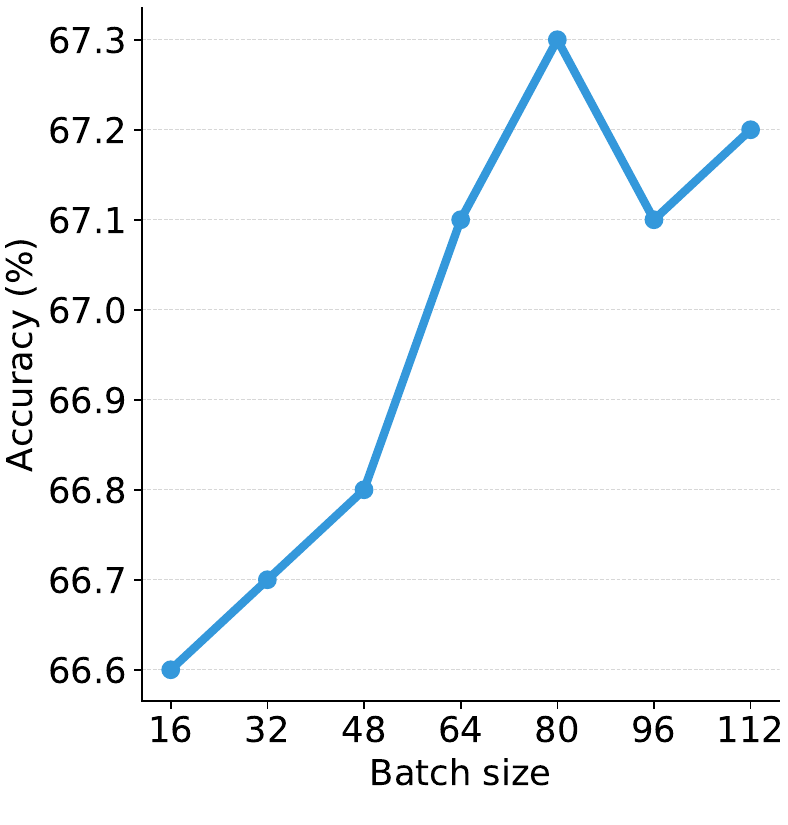}
\par\end{centering}
}\hfill{}\subfloat[TerraIncognita]{\begin{centering}
\includegraphics[width=0.23\textwidth]{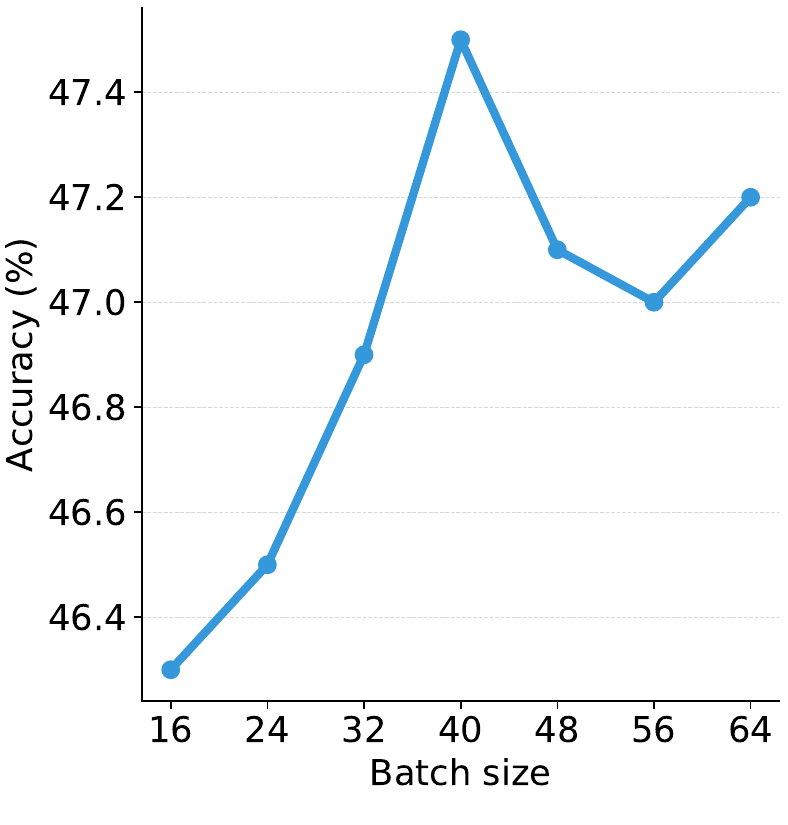}
\par\end{centering}
}\hfill{}\subfloat[DomainNet]{\begin{centering}
\includegraphics[width=0.23\textwidth]{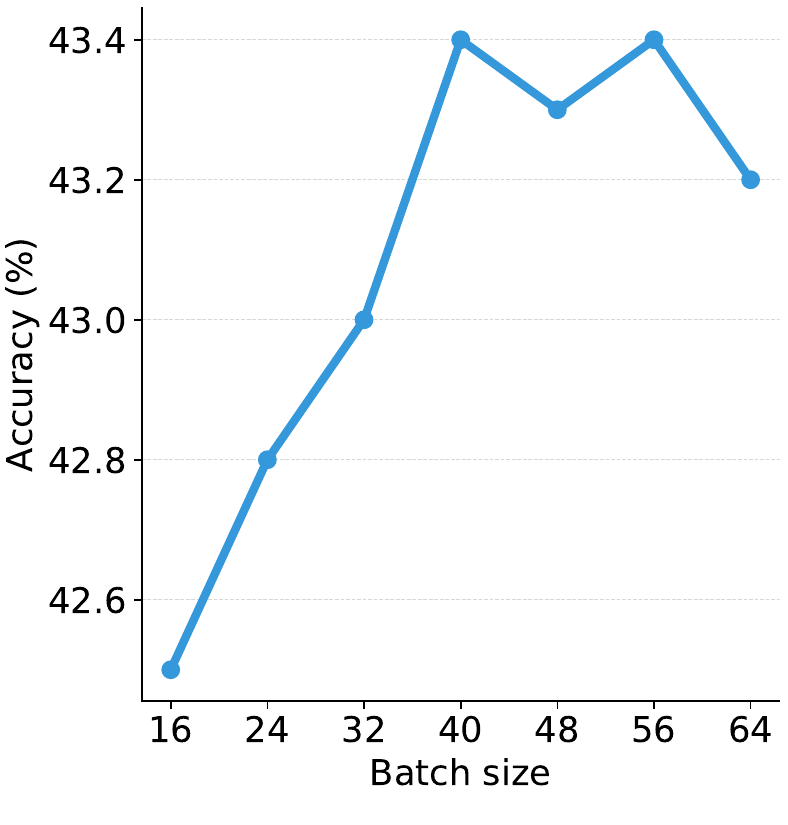}
\par\end{centering}
}

\caption{The influence of batch size in our method on VLCS, OfficeHome, TerraIncognita
and DomainNet.\label{fig:supp-batchsize-other-datasets}}
\end{figure*}

\begin{figure*}
\subfloat[VLCS]{\begin{centering}
\includegraphics[width=0.23\textwidth]{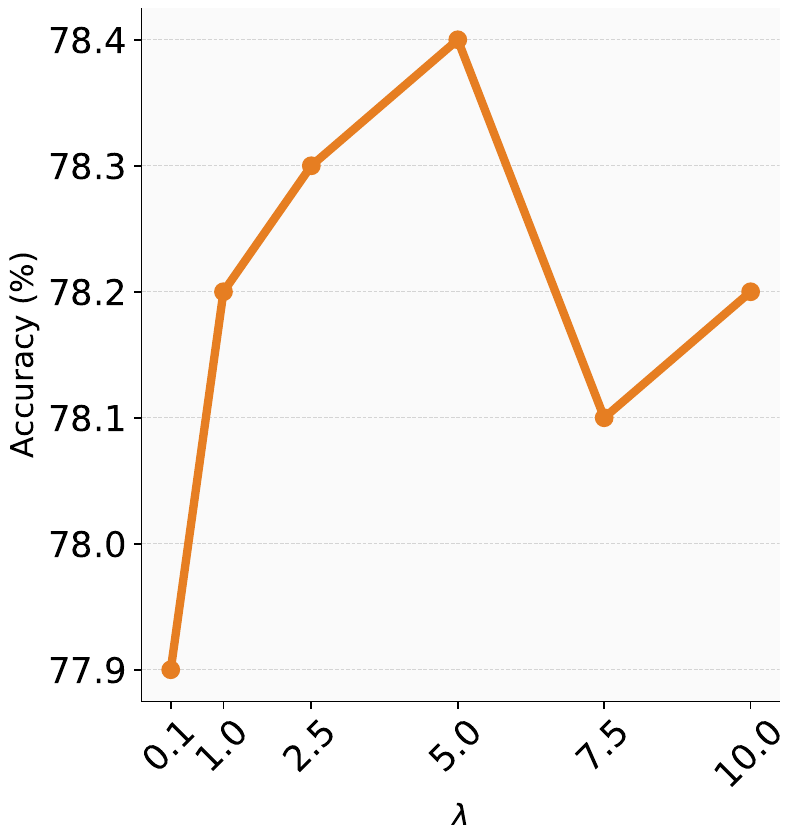}
\par\end{centering}
}\hfill{}\subfloat[OfficeHome]{\begin{centering}
\includegraphics[width=0.23\textwidth]{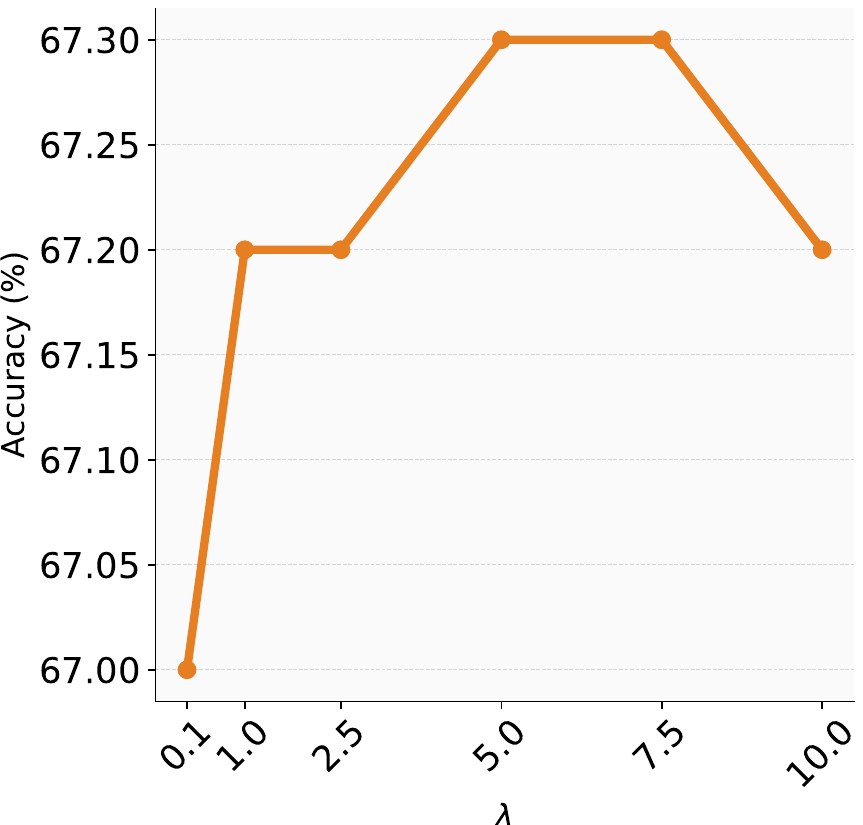}
\par\end{centering}
}\hfill{}\subfloat[TerraIncognita]{\begin{centering}
\includegraphics[width=0.23\textwidth]{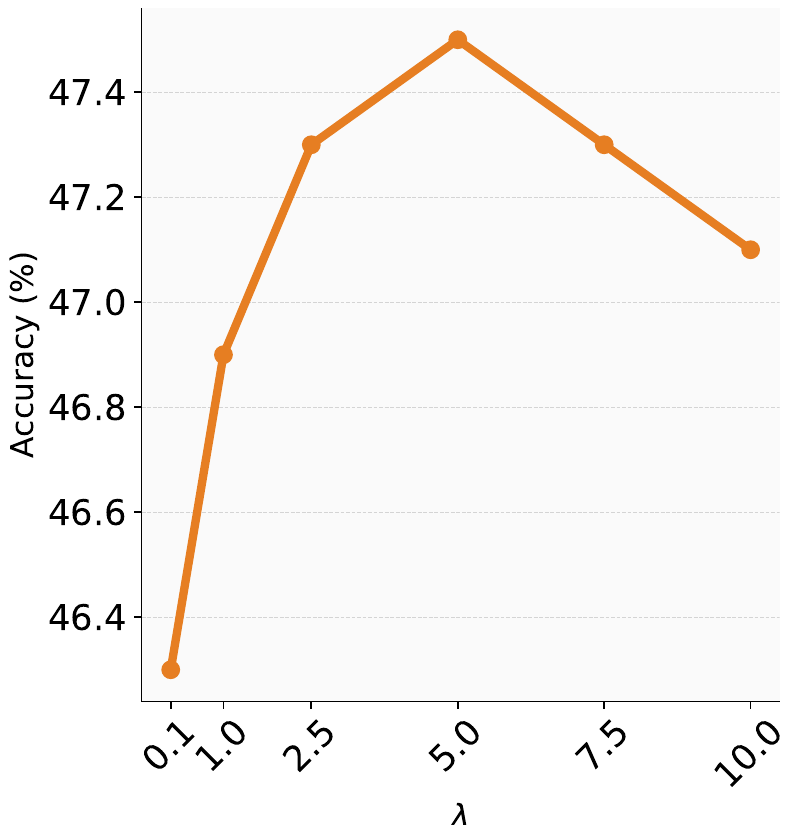}
\par\end{centering}
}\hfill{}\subfloat[DomainNet]{\begin{centering}
\includegraphics[width=0.23\textwidth]{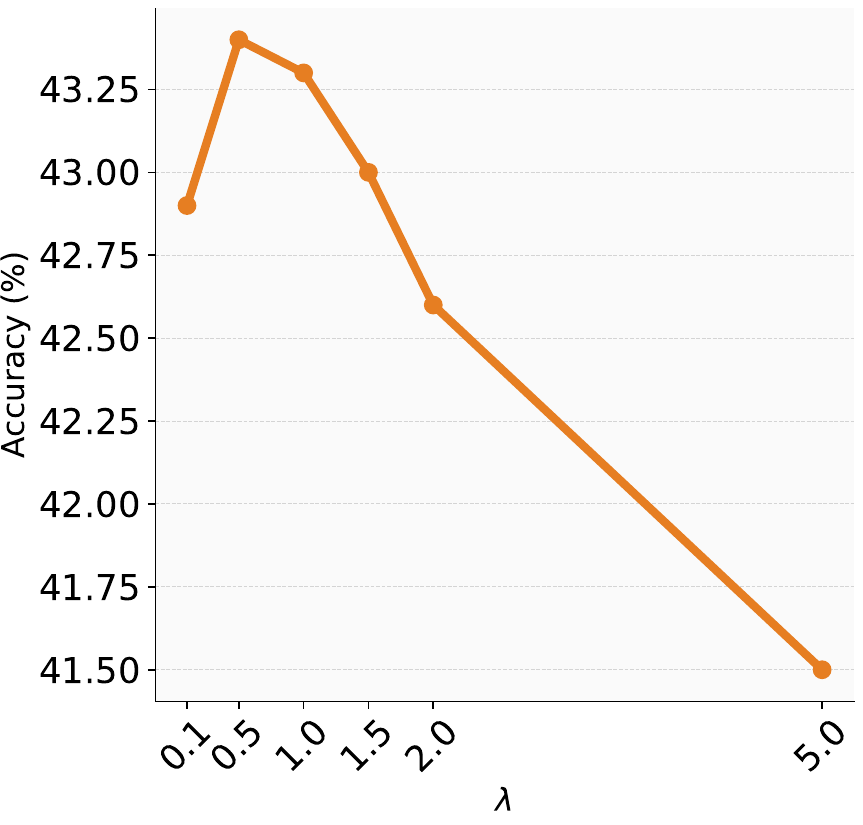}
\par\end{centering}
}

\caption{The influence of matching coefficient $\lambda$ in our method on
VLCS, OfficeHome, TerraIncognita and DomainNet.\label{fig:supp-matching-coefficient-other-datasets}}
\end{figure*}

In our analysis, we closely examine how batch size and the matching
coefficient $\lambda$ affect $\Model$'s performance across four
benchmark datasets: VLCS, OfficeHome, TerraIncognita, and DomainNet.
Consistent with our main text findings on PACS, Figure~\ref{fig:supp-batchsize-other-datasets}
shows that using larger batch sizes enhances the model's generalisation
by facilitating accurate risk distribution matching. Similarly, Figure~\ref{fig:supp-matching-coefficient-other-datasets}
highlights the importance of $\lambda$ in improving OOD performance;
as $\lambda$ increases, OOD performance generally improves. 

We find optimal batch size ranges for each dataset: VLCS and OfficeHome
perform best with sizes between $\left[70,100\right]$, while the
larger datasets of TerraIncognita and DomainNet benefit from a more
limited range of $\left[30,60\right]$. Even with computational limitations,
these batch sizes lead to strong performance. For most datasets, a
$\lambda$ value between $\left[0.1,10.0\right]$ is effective. In
the case of DomainNet, a smaller $\lambda$ range of $\left[0.1,1.0\right]$
works well, balancing the reduction of training risks and the alignment
of risk distributions across domains. This is particularly important
for large-scale datasets where reducing training risks is crucial
for learning predictive features, especially during the initial phases
of training.

\subsection{Risk distributions}

\begin{figure*}
\subfloat[ERM's histogram]{\begin{centering}
\includegraphics[width=0.49\textwidth]{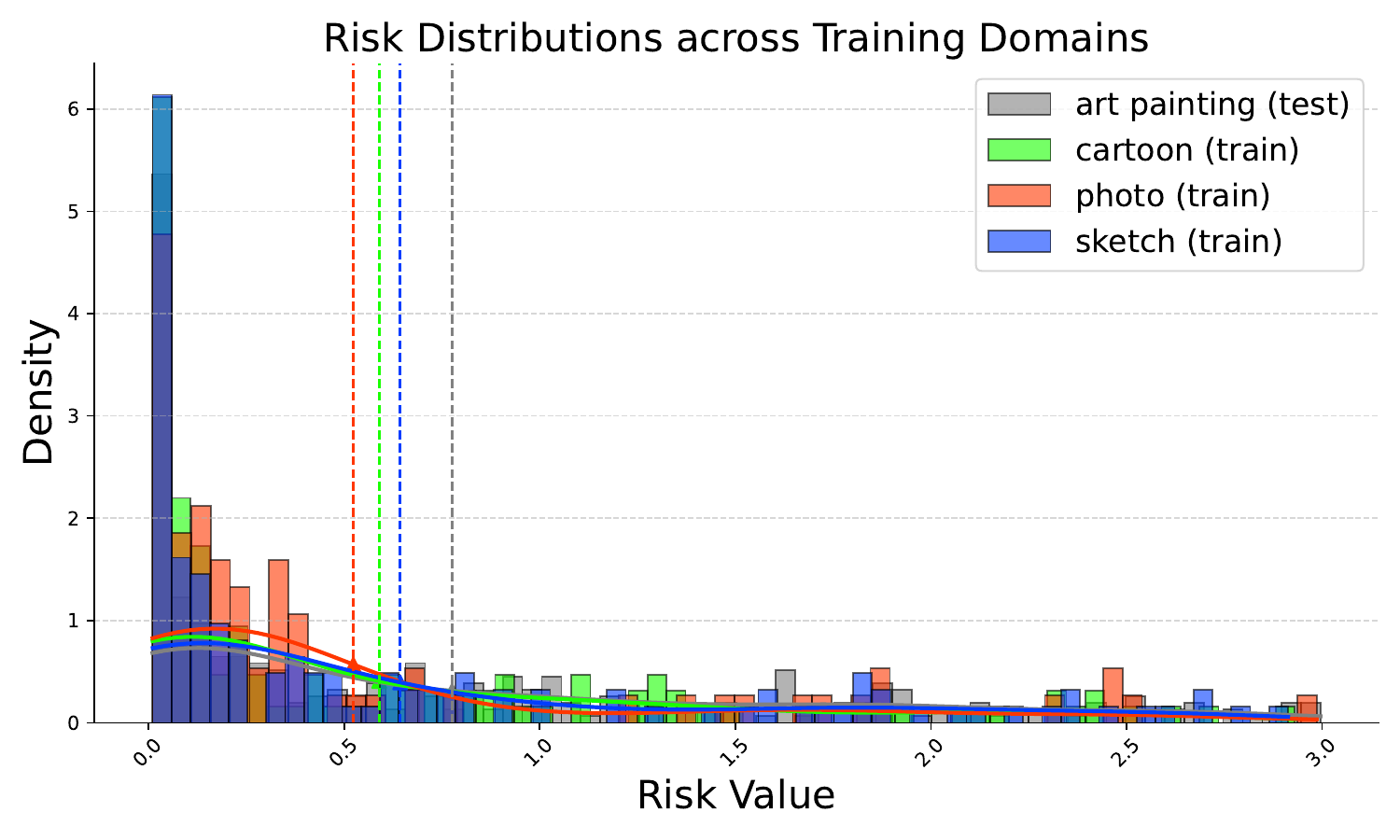}
\par\end{centering}
}\hfill{}\subfloat[$\protect\Model$'s histogram]{\begin{centering}
\includegraphics[width=0.49\textwidth]{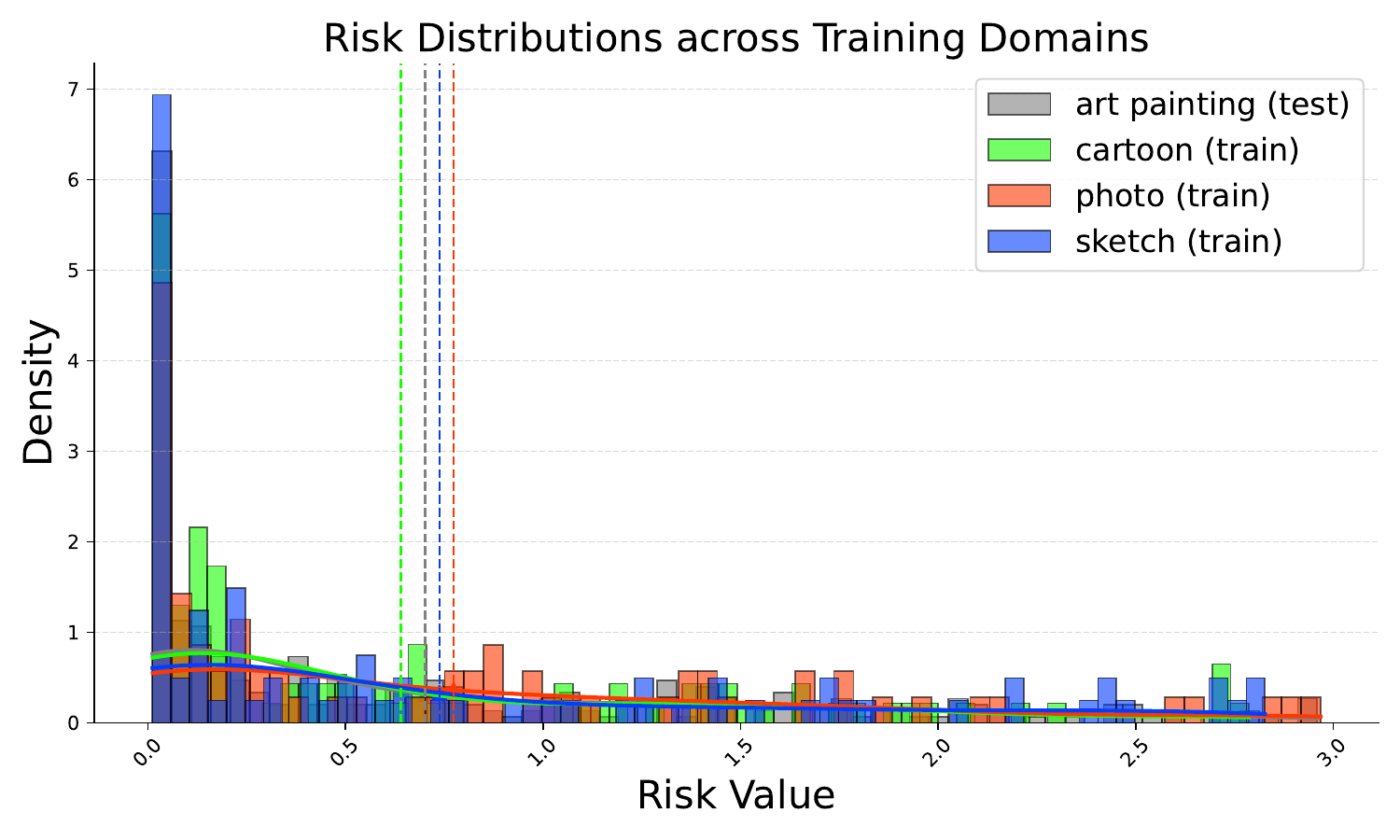}
\par\end{centering}
}

\caption{Histograms with their KDE curves depicting the risk distributions
of $\protect\ModelERM$ and our $\protect\Model$ method across four
domains on PACS. Vertical ticks denote the mean values of all distributions.\label{fig:supp-Histograms-PACS}}
\end{figure*}

\begin{figure*}
\subfloat[ERM's histogram]{\begin{centering}
\includegraphics[width=0.49\textwidth]{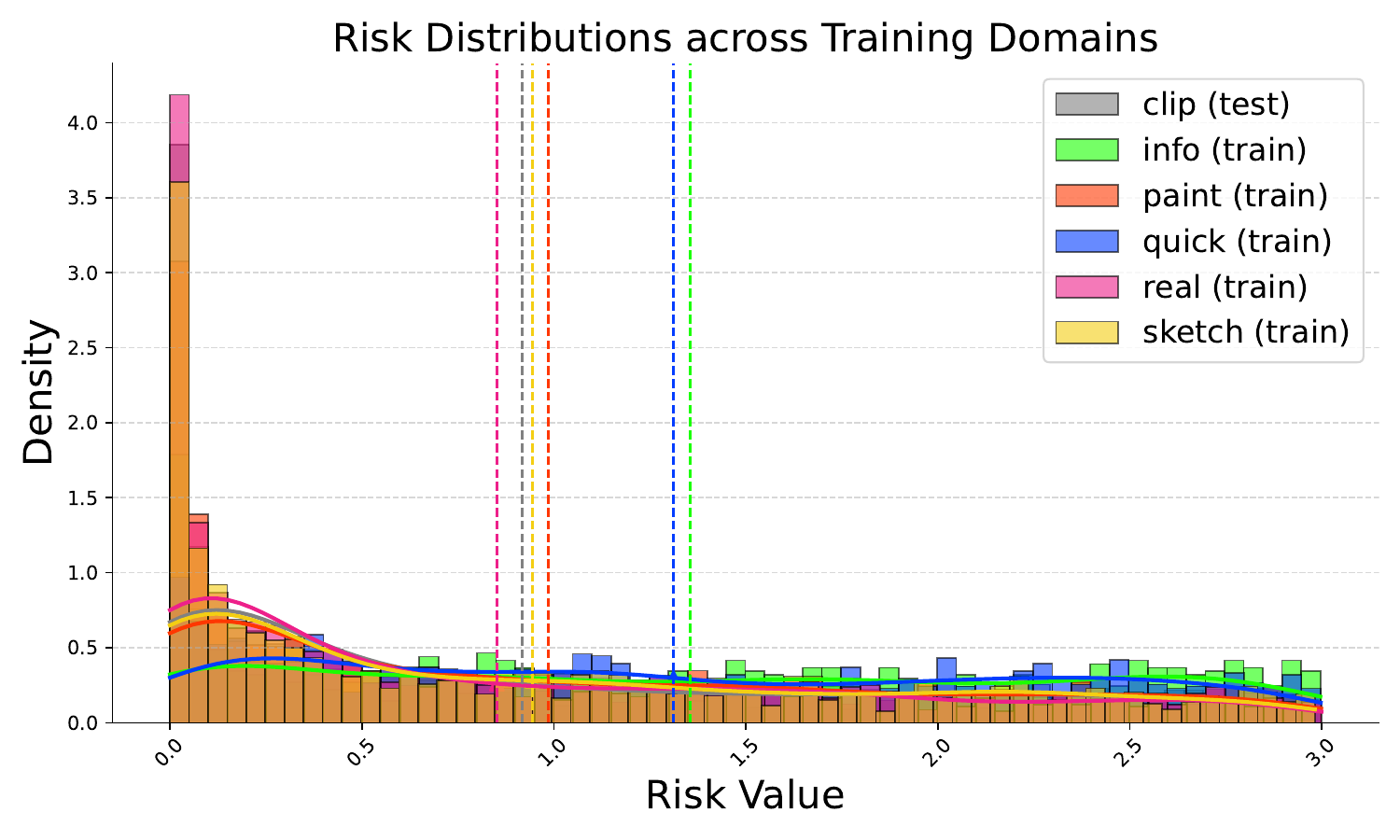}
\par\end{centering}
}\hfill{}\subfloat[$\protect\Model$'s histogram]{\begin{centering}
\includegraphics[width=0.49\textwidth]{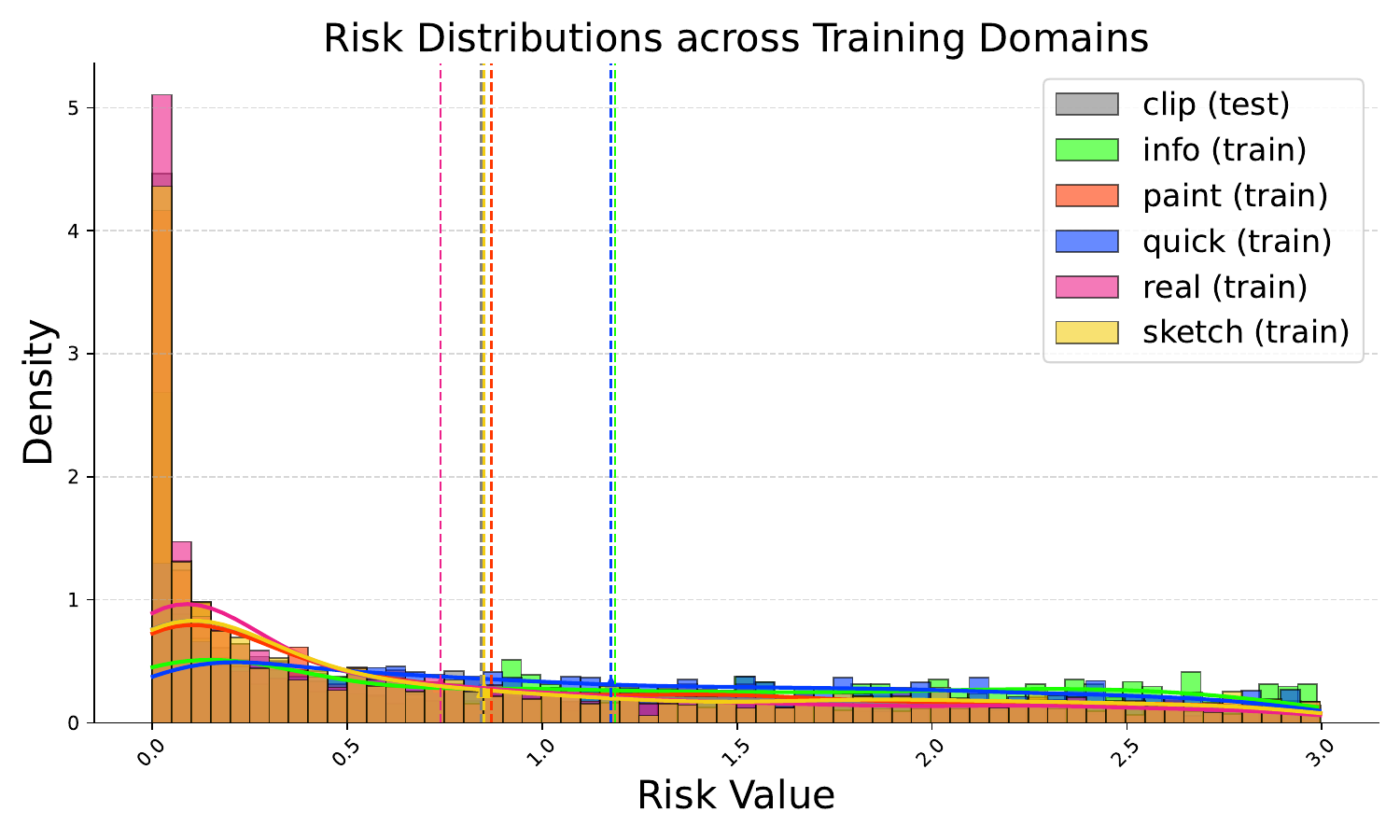}
\par\end{centering}
}

\caption{Histograms with their KDE curves depicting the risk distributions
of $\protect\ModelERM$ and our $\protect\Model$ method across six
domains on DomainNet. Vertical ticks denote the mean values of all
distributions.\label{fig:supp-Histograms-DomainNet}}
\end{figure*}

We present visualisations of risk distribution histograms accompanied
by their KDE curves for two datasets, PACS and DomainNet, in Figures~\ref{fig:supp-Histograms-PACS}
and~\ref{fig:supp-Histograms-DomainNet}, respectively. These visualisations
compare the risk distributions of $\ModelERM$ and our proposed $\Model$
method on the validation sets. Both figures confirm our hypothesis
that variations in training domains lead to distinct risk distributions,
making them valuable indicators of \emph{domain differences}.

On PACS, we observe that $\ModelERM$ tends to capture domain-specific
features, resulting in low risks within the training domains. However,
$\ModelERM$'s substantial deviation of the average risk for the test
domain from that for the training domains suggests sub-optimal OOD
generalisation. In contrast, our $\Model$ approach prioritises stable,
domain-invariant features, yielding more \emph{consistent risk distributions}
and enhanced generalisation. This trend holds across both two datasets,
as our approach consistently aligns risk distributions across domains
better than $\ModelERM$. This alignment effectively narrows the gap
between test and training domains, especially \emph{reducing risks
for test domains}.

These findings underscore the efficacy of our $\Model$ method in
mitigating domain variations by aligning risk distributions, ultimately
leading to enhanced generalisation.

\section{More experimental results\label{sec:supp-More-experimental-results}}

We provide domain-specific out-of-domain accuracies for each dataset
within the DomainBed suite in Tables~\ref{tab:Domain-specific-out-of-domain-acc-DomainNet},~\ref{tab:Domain-specific-out-of-domain-acc-PACS},~\ref{tab:Domain-specific-out-of-domain-acc-VLCS},~\ref{tab:Domain-specific-out-of-domain-acc-OfficeHome},~\ref{tab:Domain-specific-out-of-domain-acc-TerraIncognita}.
In each table, the accuracy listed in each column represents the out-of-domain
performance when that specific domain is excluded from the training
set and used solely for testing within the respective dataset. We
note that the per-domain results for Fish~\cite{shi2022gradient}
are not available.

\begin{table*}
\begin{centering}
\begin{tabular}{lccccccc}
\toprule 
\textbf{Algorithm} & \textbf{clip} & \textbf{info} & \textbf{paint} & \textbf{quick} & \textbf{real} & \textbf{sketch} & \textbf{Avg}\tabularnewline
\midrule 
ERM & 58.1 $\pm$ 0.3 & 18.8 $\pm$ 0.3 & 46.7 $\pm$ 0.3 & 12.2 $\pm$ 0.4 & 59.6 $\pm$ 0.1 & 49.8 $\pm$ 0.4 & 40.9\tabularnewline
Mixup & 55.7 $\pm$ 0.3 & 18.5 $\pm$ 0.5 & 44.3 $\pm$ 0.5 & 12.5 $\pm$ 0.4 & 55.8 $\pm$ 0.3 & 48.2 $\pm$ 0.5 & 39.2\tabularnewline
MLDG & 59.1 $\pm$ 0.2 & 19.1 $\pm$ 0.3 & 45.8 $\pm$ 0.7 & 13.4 $\pm$ 0.3 & 59.6 $\pm$ 0.2 & 50.2 $\pm$ 0.4 & 41.2\tabularnewline
GroupDRO & 47.2 $\pm$ 0.5 & 17.5 $\pm$ 0.4 & 33.8 $\pm$ 0.5 & 9.3 $\pm$ 0.3 & 51.6 $\pm$ 0.4 & 40.1 $\pm$ 0.6 & 33.3\tabularnewline
IRM & 48.5 $\pm$ 2.8 & 15.0 $\pm$ 1.5 & 38.3 $\pm$ 4.3 & 10.9 $\pm$ 0.5 & 48.2 $\pm$ 5.2 & 42.3 $\pm$ 3.1 & 33.9\tabularnewline
VREx & 47.3 $\pm$ 3.5 & 16.0 $\pm$ 1.5 & 35.8 $\pm$ 4.6 & 10.9 $\pm$ 0.3 & 49.6 $\pm$ 4.9 & 42.0 $\pm$ 3.0 & 33.6\tabularnewline
EQRM & 56.1 $\pm$ 1.3 & 19.6 $\pm$ 0.1 & 46.3 $\pm$ 1.5 & 12.9 $\pm$ 0.3 & 61.1 $\pm$ 0.0 & 50.3 $\pm$ 0.1 & 41.0\tabularnewline
Fish & - & - & - & - & - & - & 42.7\tabularnewline
Fishr & 58.2 $\pm$ 0.5 & 20.2 $\pm$ 0.2 & 47.7 $\pm$ 0.3 & 12.7 $\pm$ 0.2 & 60.3 $\pm$ 0.2 & 50.8 $\pm$ 0.1 & 41.7\tabularnewline
CORAL & 59.2 $\pm$ 0.1 & 19.7 $\pm$ 0.2 & 46.6 $\pm$ 0.3 & 13.4 $\pm$ 0.4 & 59.8 $\pm$ 0.2 & 50.1 $\pm$ 0.6 & 41.5\tabularnewline
MMD & 32.1 $\pm$ 13.3 & 11.0 $\pm$ 4.6 & 26.8 $\pm$ 11.3 & 8.7 $\pm$ 2.1 & 32.7 $\pm$ 13.8 & 28.9 $\pm$ 11.9 & 23.4\tabularnewline
\midrule 
RDM (\emph{ours}) & \textbf{62.1 $\pm$ 0.2} & \textbf{20.7 $\pm$ 0.1} & \textbf{49.2 $\pm$ 0.4} & \textbf{14.1 $\pm$ 0.4} & \textbf{63.0 $\pm$ 1.3} & \textbf{51.4 $\pm$ 0.1} & \textbf{43.4}\tabularnewline
\bottomrule
\end{tabular}
\par\end{centering}
\caption{Domain-specific out-of-domain accuracy on DomainNet where the best
results are marked as bold. Results of other methods are referenced
from~\cite{eastwood2022probable,shi2022gradient}.\label{tab:Domain-specific-out-of-domain-acc-DomainNet}}
\end{table*}

\begin{table*}
\begin{centering}
\begin{tabular}{lccccc}
\toprule 
\textbf{Algorithm} & \textbf{A} & \textbf{C} & \textbf{P} & \textbf{S} & \textbf{Avg}\tabularnewline
\midrule 
ERM & 84.7 $\pm$ 0.4 & 80.8 $\pm$ 0.6 & 97.2 $\pm$ 0.3 & 79.3 $\pm$ 1.0 & 85.5\tabularnewline
Mixup & 86.1 $\pm$ 0.5 & 78.9 $\pm$ 0.8 & \textbf{97.6 $\pm$ 0.1} & 75.8 $\pm$ 1.8 & 84.6\tabularnewline
MLDG & 85.5 $\pm$ 1.4 & 80.1 $\pm$ 1.7 & 97.4 $\pm$ 0.3 & 76.6 $\pm$ 1.1 & 84.9\tabularnewline
GroupDRO & 83.5 $\pm$ 0.9 & 79.1 $\pm$ 0.6 & 96.7 $\pm$ 0.3 & 78.3 $\pm$ 2.0 & 84.4\tabularnewline
IRM & 84.8 $\pm$ 1.3 & 76.4 $\pm$ 1.1 & 96.7 $\pm$ 0.6 & 76.1 $\pm$ 1.0 & 83.5\tabularnewline
VREx & 86.0 $\pm$ 1.6 & 79.1 $\pm$ 0.6 & 96.9 $\pm$ 0.5 & 77.7 $\pm$ 1.7 & 84.9\tabularnewline
EQRM & 86.5 $\pm$ 0.4 & \textbf{82.1 $\pm$ 0.7} & 96.6 $\pm$ 0.2 & 80.8 $\pm$ 0.2 & 86.5\tabularnewline
Fish & - & - & - & - & 85.5\tabularnewline
Fishr & \textbf{88.4 $\pm$ 0.2} & 78.7$\pm$ 0.7 & 97.0 $\pm$ 0.1 & 77.8 $\pm$ 2.0 & 85.5\tabularnewline
CORAL & 88.3 $\pm$ 0.2 & 80.0 $\pm$ 0.5 & 97.5 $\pm$ 0.3 & 78.8 $\pm$ 1.3 & 86.2\tabularnewline
MMD & 86.1 $\pm$ 1.4 & 79.4 $\pm$ 0.9 & 96.6 $\pm$ 0.2 & 76.5 $\pm$ 0.5 & 84.6\tabularnewline
\midrule 
RDM (\emph{ours}) & \textbf{88.4 $\pm$ 0.2} & 81.3 $\pm$ 1.6 & 97.1 $\pm$ 0.1 & \textbf{81.8 $\pm$ 1.1} & \textbf{87.2}\tabularnewline
\bottomrule
\end{tabular}
\par\end{centering}
\caption{Domain-specific out-of-domain accuracy on PACS where the best results
are marked as bold. Results of other methods are referenced from~\cite{eastwood2022probable,shi2022gradient}.\label{tab:Domain-specific-out-of-domain-acc-PACS}}
\end{table*}

\begin{table*}
\begin{centering}
\begin{tabular}{lccccc}
\toprule 
\textbf{Algorithm} & \textbf{C} & \textbf{L} & \textbf{S} & \textbf{V} & \textbf{Avg}\tabularnewline
\midrule 
ERM & 97.7 $\pm$ 0.4 & 64.3 $\pm$ 0.9 & 73.4 $\pm$ 0.5 & 74.6 $\pm$ 1.3 & 77.5\tabularnewline
Mixup & 98.3 $\pm$ 0.6 & 64.8 $\pm$ 1.0 & 72.1 $\pm$ 0.5 & 74.3 $\pm$ 0.8 & 77.4\tabularnewline
MLDG & 97.4 $\pm$ 0.2 & 65.2 $\pm$ 0.7 & 71.0 $\pm$ 1.4 & 75.3 $\pm$ 1.0 & 77.2\tabularnewline
GroupDRO & 97.3 $\pm$ 0.3 & 63.4 $\pm$ 0.9 & 69.5 $\pm$ 0.8 & 76.7 $\pm$ 0.7 & 76.7\tabularnewline
IRM & 98.6 $\pm$ 0.1 & 64.9 $\pm$ 0.9 & 73.4 $\pm$ 0.6 & 77.3 $\pm$ 0.9 & 78.5\tabularnewline
VREx & 98.4 $\pm$ 0.3 & 64.4 $\pm$ 1.4 & \textbf{74.1 $\pm$ 0.4} & 76.2 $\pm$ 1.3 & 78.3\tabularnewline
EQRM & 98.3 $\pm$ 0.0 & 63.7 $\pm$ 0.8 & 72.6 $\pm$ 1.0 & 76.7 $\pm$ 1.1 & 77.8\tabularnewline
Fish & - & - & - & - & 77.8\tabularnewline
Fishr & \textbf{98.9 $\pm$ 0.3} & 64.0 $\pm$ 0.5 & 71.5 $\pm$ 0.2 & 76.8 $\pm$ 0.7 & 77.8\tabularnewline
CORAL & 98.3 $\pm$ 0.1 & \textbf{66.1 $\pm$ 1.2} & 73.4 $\pm$ 0.3 & 77.5 $\pm$ 1.2 & \textbf{78.8}\tabularnewline
MMD & 97.7 $\pm$ 0.1 & 64.0 $\pm$ 1.1 & 72.8 $\pm$ 0.2 & 75.3 $\pm$ 3.3 & 77.5\tabularnewline
\midrule 
RDM (\emph{ours}) & 98.1 $\pm$ 0.2 & 64.9 $\pm$ 0.7 & 72.6 $\pm$ 0.5 & \textbf{77.9 $\pm$ 1.2} & 78.4\tabularnewline
\bottomrule
\end{tabular}
\par\end{centering}
\caption{Domain-specific out-of-domain accuracy on VLCS where the best results
are marked as bold. Results of other methods are referenced from~\cite{eastwood2022probable,shi2022gradient}.\label{tab:Domain-specific-out-of-domain-acc-VLCS}}
\end{table*}

\begin{table*}
\begin{centering}
\begin{tabular}{lccccc}
\toprule 
\textbf{Algorithm} & \textbf{A} & \textbf{C} & \textbf{P} & \textbf{R} & \textbf{Avg}\tabularnewline
\midrule 
ERM & 61.3 $\pm$ 0.7 & 52.4 $\pm$ 0.3 & 75.8 $\pm$ 0.1 & 76.6 $\pm$ 0.3 & 66.5\tabularnewline
Mixup & 62.4 $\pm$ 0.8 & 54.8 $\pm$ 0.6 & \textbf{76.9 $\pm$ 0.3} & 78.3 $\pm$ 0.2 & 68.1\tabularnewline
MLDG & 61.5 $\pm$ 0.9 & 53.2 $\pm$ 0.6 & 75.0 $\pm$ 1.2 & 77.5 $\pm$ 0.4 & 66.8\tabularnewline
GroupDRO & 60.4 $\pm$ 0.7 & 52.7 $\pm$ 1.0 & 75.0 $\pm$ 0.7 & 76.0 $\pm$ 0.7 & 66.0\tabularnewline
IRM & 58.9 $\pm$ 2.3 & 52.2 $\pm$ 1.6 & 72.1 $\pm$ 2.9 & 74.0 $\pm$ 2.5 & 64.3\tabularnewline
VREx & 60.7 $\pm$ 0.9 & 53.0 $\pm$ 0.9 & 75.3 $\pm$ 0.1 & 76.6 $\pm$ 0.5 & 66.4\tabularnewline
EQRM & 60.5 $\pm$ 0.1 & \textbf{56.0 $\pm$ 0.2} & 76.1 $\pm$ 0.4 & 77.4 $\pm$ 0.3 & 67.5\tabularnewline
Fish & - & - & - & - & 68.6\tabularnewline
Fishr & 62.4 $\pm$ 0.5 & 54.4 $\pm$ 0.4 & 76.2 $\pm$ 0.5 & 78.3 $\pm$ 0.1 & 67.8\tabularnewline
CORAL & \textbf{65.3 $\pm$ 0.4} & 54.4 $\pm$ 0.5 & 76.5 $\pm$ 0.1 & \textbf{78.4 $\pm$ 0.5} & \textbf{68.7}\tabularnewline
MMD & 60.4 $\pm$ 0.2 & 53.3 $\pm$ 0.3 & 74.3 $\pm$ 0.1 & 77.4 $\pm$ 0.6 & 66.3\tabularnewline
\midrule 
RDM (\emph{ours}) & 61.1 $\pm$ 0.4 & 55.1 $\pm$ 0.3 & 75.7 $\pm$ 0.5 & 77.3 $\pm$ 0.3 & 67.3\tabularnewline
\bottomrule
\end{tabular}
\par\end{centering}
\caption{Domain-specific out-of-domain accuracy on OfficeHome where the best
results are marked as bold. Results of other methods are referenced
from~\cite{eastwood2022probable,shi2022gradient}.\label{tab:Domain-specific-out-of-domain-acc-OfficeHome}}
\end{table*}

\begin{table*}[t]
\begin{centering}
\begin{tabular}{lccccc}
\toprule 
\textbf{Algorithm} & \textbf{L100} & \textbf{L38} & \textbf{L43} & \textbf{L46} & \textbf{Avg}\tabularnewline
\midrule 
ERM & 49.8 $\pm$ 4.4 & 42.1 $\pm$ 1.4 & 56.9 $\pm$ 1.8 & 35.7 $\pm$ 3.9 & 46.1\tabularnewline
Mixup & \textbf{59.6 $\pm$ 2.0} & 42.2 $\pm$ 1.4 & 55.9 $\pm$ 0.8 & 33.9 $\pm$ 1.4 & \textbf{47.9}\tabularnewline
MLDG & 54.2 $\pm$ 3.0 & 44.3 $\pm$ 1.1 & 55.6 $\pm$ 0.3 & 36.9 $\pm$ 2.2 & 47.7\tabularnewline
GroupDRO & 41.2 $\pm$ 0.7 & 38.6 $\pm$ 2.1 & 56.7 $\pm$ 0.9 & 36.4 $\pm$ 2.1 & 43.2\tabularnewline
IRM & 54.6 $\pm$ 1.3 & 39.8 $\pm$ 1.9 & 56.2 $\pm$ 1.8 & 39.6 $\pm$ 0.8 & 47.6\tabularnewline
VREx & 48.2 $\pm$ 4.3 & 41.7 $\pm$ 1.3 & 56.8 $\pm$ 0.8 & 38.7 $\pm$ 3.1 & 46.4\tabularnewline
EQRM & 47.9 $\pm$ 1.9 & \textbf{45.2 $\pm$ 0.3} & \textbf{59.1 $\pm$ 0.3} & 38.8 $\pm$ 0.6 & 47.8\tabularnewline
Fish & - & - & - & - & 45.1\tabularnewline
Fishr & 50.2 $\pm$ 3.9 & 43.9 $\pm$ 0.8 & 55.7 $\pm$ 2.2 & \textbf{39.8 $\pm$ 1.0} & 47.4\tabularnewline
CORAL & 51.6 $\pm$ 2.4 & 42.2 $\pm$ 1.0 & 57.0 $\pm$ 1.0 & \textbf{39.8 $\pm$ 2.9} & 47.6\tabularnewline
MMD & 41.9 $\pm$ 3.0 & 34.8 $\pm$ 1.0 & 57.0 $\pm$ 1.9 & 35.2 $\pm$ 1.8 & 42.2\tabularnewline
\midrule 
RDM (\emph{ours}) & 52.9 $\pm$ 1.2 & 43.1 $\pm$ 1.0 & 58.1 $\pm$ 1.3 & 36.1 $\pm$ 2.9 & 47.5\tabularnewline
\bottomrule
\end{tabular}
\par\end{centering}
\caption{Domain-specific out-of-domain accuracy on TerraIncognita where the
best results are marked as bold. Results of other methods are referenced
from~\cite{eastwood2022probable,shi2022gradient}.\label{tab:Domain-specific-out-of-domain-acc-TerraIncognita}}
\end{table*}

\end{document}